\documentclass[11pt]{article}

\usepackage[preprint]{acl}

\usepackage{times}
\usepackage{latexsym}

\usepackage[T1]{fontenc}

\usepackage[utf8]{inputenc}

\usepackage{microtype}

\usepackage{inconsolata}

\usepackage{graphicx}

\usepackage{xcolor}

\definecolor{hanblue}{HTML}{1581CC}
\definecolor{hanred}{HTML}{A61E32}
\definecolor{hanyellow}{HTML}{F8B72B}

\newcommand{\colorboxlabel}[2]{%
    \protect\begin{tikzpicture}[baseline=(char.base)]
        \protect\node[
            fill=#1, 
            shape=rectangle, 
            inner sep=0.6pt,          
            text=white, 
            font=\footnotesize\sffamily
        ] (char) {#2};
    \protect\end{tikzpicture}%
}

\usepackage[most]{tcolorbox}
\usepackage{fontawesome5} 
\definecolor{main-blue}{RGB}{56, 107, 187} 
\newtcolorbox{takeaway}[1][]{
    enhanced,
    boxrule=0pt,frame hidden, 
    borderline west={4pt}{0pt}{main-blue}, 
    colback=main-blue!10, 
    coltitle=main-blue, 
    coltext=main-blue,  
    fonttitle=\bfseries\sffamily, 
    fontupper=\small\sffamily,
    attach boxed title to top left={yshift=-10pt,xshift=10pt}, 
    before upper={\setlength{\baselineskip}{1.2\baselineskip}}, 
    boxed title style={boxrule=0pt,frame hidden,colback=main-blue!10}, 
    title={\faLightbulb\ Takeaway}, 
    #1
}

\newtcolorbox{summarybox}[1][]{
    enhanced,
    colback=white, 
    colframe=main-blue, 
    fonttitle=\bfseries\sffamily\Large,
    coltitle=white,
    title={Key Takeaway}, 
    sharp corners, 
    drop shadow, 
    #1
}

\newtcolorbox{minimalbox}[1][]{
    enhanced,
    frame hidden, 
    colback=white, 
    borderline north={1pt}{0pt}{main-blue}, 
    borderline south={1pt}{0pt}{main-blue}, 
    fonttitle=\bfseries\sffamily\color{main-blue},
    title={Takeaway:},
    attach boxed title to top left={yshift=-10pt},
    boxed title style={frame hidden, colback=white},
    left=0pt, right=0pt, 
    #1
}

\usepackage[normalem]{ulem}
\usepackage{makecell} 
\usepackage{colortbl} 
\usepackage{cleveref}
\usepackage{scalerel} 
\usepackage{pifont}
\usepackage{tcolorbox}
\usepackage[framemethod=TikZ]{mdframed}

\crefname{appendix}{Appendix}{Appendices}
\Crefname{appendix}{Appendix}{Appendices}

\definecolor{hanred}{RGB}{165,30,50}
\newcommand{\hanredcell}[1]{\cellcolor{hanred!60}{#1}}
\definecolor{sudisblue}{RGB}{93, 177, 251}
\definecolor{sudisred}{RGB}{252,106,108}
\hypersetup{
    breaklinks,
    citecolor=main-blue,
    colorlinks=true,
    linkcolor=main-blue,
    urlcolor=sudisred
}

\usepackage{booktabs}       %
\usepackage{amsfonts}       %
\usepackage{nicefrac}       %
\usepackage{enumitem}
\usepackage{colortbl}
\usepackage[edges]{forest}
\usepackage{adjustbox}
\usepackage{xcolor}
\usepackage{ragged2e}

\usepackage{multirow}
\usepackage{graphicx}
\usepackage[table]{xcolor}
\usepackage{xspace}
\usepackage{makecell}

\newcommand{\method}{\textsc{LVSpec}\xspace}

\newcommand{\gray}[1]{\cellcolor{gray!10}\textcolor{gray}{#1}}

\usepackage{amsmath}
\usepackage{amssymb}
\usepackage{amsthm}
\usepackage{mathtools}
\theoremstyle{plain}
\newtheorem{theorem}{Theorem}[section]
\newtheorem{proposition}[theorem]{Proposition}

\theoremstyle{definition}

\theoremstyle{remark}
\newtheorem{remark}[theorem]{Remark}

\usepackage{xcolor}
\usepackage{tikz}
\usepackage{pdfrender}
\usepackage{xspace}
\usepackage{hyperref}

\newsavebox{\gradienttextbox}
\newcommand{\gradienttext}[3]{%
  \begingroup
  \sbox{\gradienttextbox}{#3}%
  \mbox{%
    \pdfsave
    \pdfrender{TextRenderingMode=Clip}%
    \rlap{\usebox{\gradienttextbox}}%
    \rlap{%
      \raisebox{-\dp\gradienttextbox}{%
        \begin{tikzpicture}[x=1pt,y=1pt,inner sep=0pt,outer sep=0pt]
          \shade[left color=#1, right color=#2]
            (0,0) rectangle (\wd\gradienttextbox,\ht\gradienttextbox+\dp\gradienttextbox);
        \end{tikzpicture}%
      }%
    }%
    \pdfrestore
    \phantom{\usebox{\gradienttextbox}}%
  }%
  \endgroup
}
\DeclareRobustCommand{\methodcolor}{%
  \gradienttext{hanred}{sudisblue}{\textsc{See the Forest for the Trees}}%
}

\title{\texorpdfstring
  {\methodcolor: Loosely Speculative Decoding via Visual-Semantic Guidance for Efficient Inference of Video LLMs}
}

\author{
\textbf{Yicheng Ji}$^{1\dagger}$ \quad
\textbf{Jun Zhang}$^{1\dagger}$ \quad
\textbf{Jinpeng Chen}$^{2}$ \\
\textbf{Cong Wang}$^{1}$ \quad
\textbf{Lidan Shou}$^{1}$ \quad
\textbf{Gang Chen}$^{1}$ \quad
\textbf{Huan Li}$^{1*}$ \\
$^{1}$ZJU \quad $^{2}$BUPT \\
\vspace{0.5em}
{\href{https://github.com/zju-jiyicheng/LVSpec}{\color{hanred}{\texttt{https://github.com/zju-jiyicheng/LVSpec}}}}
}

\begin{document}
\maketitle

\let\oldthefootnote\thefootnote
\renewcommand{\thefootnote}{}

\footnotemark

\footnotetext{{$\dagger$} Equal contribution \qquad
{$*$} Corresponding author}

\begin{abstract}
Video Large Language Models (Video-LLMs) excel in video understanding but suffer from high inference latency during autoregressive generation. Speculative Decoding (SD) mitigates this by applying a draft-and-verify paradigm, yet existing methods are constrained by rigid \textit{exact-match} rules, severely limiting the acceleration potential.
To bridge this gap, we propose \method, the first training-free loosely SD framework tailored for Video-LLMs.
Grounded in the insight that generation is governed by sparse visual-relevant anchors (mandating strictness) amidst abundant visual-irrelevant fillers (permitting loose verification), \method employs a lightweight visual-relevant token identification scheme to accurately pinpoint the former.
To further maximize acceptance, we augment this with a position-shift tolerant mechanism that effectively salvages positionally mismatched but semantically equivalent tokens.
Experiments demonstrate that \method achieves high fidelity and speed: it preserves \textbf{>99.8\%} of target performance while accelerating \texttt{Qwen2.5-VL-32B} by \textbf{2.70}$\times$ and \texttt{LLaVA-OneVision-72B} by \textbf{2.94}$\times$. Notably, it boosts the mean accepted length and speedup ratio by \textbf{136\%} and \textbf{35\%} compared to SOTA training-free SD methods for Video-LLMs.
\end{abstract}
\section{Introduction}

Video large language models (Video-LLMs)~\citep{DBLP:conf/nips/LiuLWL23a,DBLP:journals/corr/abs-2308-12966,DBLP:conf/cvpr/LiuLLL24,DBLP:conf/emnlp/LinYZCNJ024,DBLP:journals/corr/abs-2407-07895,DBLP:journals/corr/abs-2502-13923,DBLP:journals/tmlr/0080ZGZ00ZZL0L25,DBLP:journals/corr/abs-2509-23661} show strong performance in video understanding tasks such as video captioning and video question answering. However,  inference latency remains a key bottleneck due to autoregressive backbones, hindering practical deployment.
For instance, to process a 100-second high-resolution video at 32 FPS, \texttt{Qwen2.5-VL}~\cite{DBLP:journals/corr/abs-2502-13923} by default encodes over one million visual tokens. The long sequence of tokens, together with the full model parameters, requires autoregressive access during decoding, leading to a memory-bound bottleneck and increasing end-to-end latency.

\begin{figure}[!t]
  \includegraphics[width=\columnwidth]{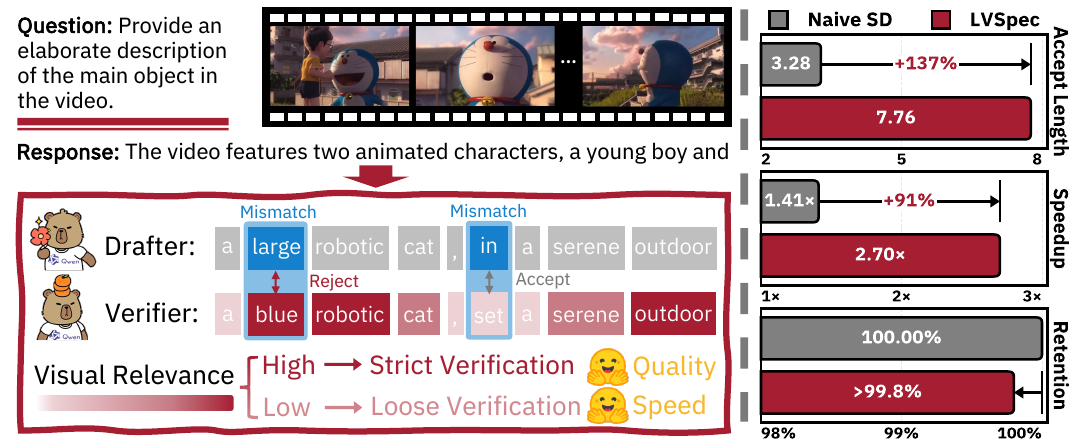}
  \caption{\method perform strict verification for visual-relevant tokens and loose verification for visual-irrlevant ones, boosting efficiency while preserving performance.}
  \label{fig:intro}
\end{figure}

To achieve lossless decoding-time acceleration, \textit{Speculative Decoding (SD)} offers a promising alternative. 
It leverages a lightweight \textit{draft model} to propose multiple draft tokens, which are then verified in parallel by a \textit{target model}.
Recent works~\citep{ji2025specvlm,DBLP:journals/corr/abs-2509-15235,DBLP:journals/corr/abs-2505-10526,DBLP:journals/corr/abs-2509-11961,DBLP:journals/corr/abs-2510-22641,DBLP:journals/corr/abs-2509-23928,DBLP:journals/corr/abs-2509-11815} have explored SD for large vision-language models, and SpecVLM~\cite{ji2025specvlm} has pioneered the application of SD to Video-LLMs in a training-free manner. However, existing SD frameworks are fundamentally constrained by their \textit{exact-match} rule: a draft token is accepted only if it is identical to the target model's generation. Such a strictly SD scheme forces the unfavorable rejection of many \emph{semantically aligned} tokens, leading to limited mean accepted lengths and speedup ratios. 

For large language models (LLMs), loosely speculative decoding has evolved from costly classifiers~\cite{BachmannAPGSDST25} to training-free heuristics like \textsc{FLy}~\cite{li2025looselysd}. Specifically, \textsc{FLy} accepts a mismatched token as a semantically equivalent alternative only when model uncertainty is high (gauged by entropy) and the divergence proves transient within a deferred window. Consequently, this approach achieves highly generalizable acceleration with minimal overhead.
However, naively adapting these advanced loosely SD tailored for LLMs to Video-LLMs suffers from a critical ``visual blindness'' limitation. By relying solely on linguistic priors, such as output entropy, these methods apply an indiscriminate verification standard, treating visually critical descriptors and trivial syntactic fillers with the same logic. 
We posit that for Video-LLMs, generation quality is defined not by rote lexical matching, but by the fidelity of articulating key visual information. This intuition raises a pivotal question:

\textit{Are all decoded tokens equally critical for visual understanding, thus mandating uniform verification rigor?} 

Our answer is negative. We assert that visual grounding must dictate verification strictness, as shown in~\Cref{fig:intro}. On one hand, tokens strongly tied to visual content (e.g., ``blue'') act as the anchors of understanding; treating them loosely (as text methods might do for high-entropy synonyms) risks factual hallucinations. Therefore, they mandate strict validation. On the other hand, visually agnostic tokens (e.g., ``set'') contribute minimally to fidelity; maintaining strict matching here (as text methods do for low-entropy non-synonyms) is computationally wasteful. Consequently, they permit adaptive and relaxed matching. Driven by this insight, we introduce \method, a novel loosely SD framework tailored for Video-LLMs that utilizes visual-semantic guidance.

\method introduces a \textbf{visual-relevant text token identification scheme} in~\Cref{section:method_1} that computes the visual similarity of text tokens at each decoding step, which accurately identifies the key text tokens that carry the richest visual cues. Subsequently, we adopt a \textbf{visual–semantic guided loose verification} in~\Cref{section:method_2} that selectively token matching constraint based on the above quantified visual relevance. 
Moreover, we introduce a \textbf{position shift-tolerant mechanism} in~\Cref{section:method_3} that relaxes acceptance for tokens that appear in the draft sequence but mismatch due to positional shifts, thereby further improving speedup without reducing the conveyed visual information.

To summarize, our main contributions are:

\begin{itemize}[itemsep=0pt, topsep=0pt, leftmargin=19pt]
    \item[(1)] \textit{Novel Discovery and Theoretical Potential}: We demystify the dual nature of visual dependence: generation is governed by sparse yet critical visual-relevant anchors amidst a sea of visual-irrelevant linguistic filler.
    Capitalizing on this, we provide a rigorous theoretical analysis to quantify the acceleration potential, guaranteeing strict verification for anchors while aggressively relaxing the rest.

    \item[(2)] \textit{Visual-Semantic Guided Loosely SD}: To unleash this potential, we propose \method, the first loosely SD tailored for Video-LLMs. By employing a pinpoint and lightweight visual-relevant token identification and loose verification scheme, augmented by a position-shift tolerant mechanism, \method shatters the rigid “exact-match” barrier, effectively salvaging rejected mismatches into valid semantic equivalents in a training-free manner.

    \item[(3)] \textit{Extensive Experiments and Analysis}: Comprehensive evaluations across four video understanding benchmarks show that \method is both high-fidelity and rapid: it preserves \textbf{>99.8\%} of target performance while accelerating \texttt{Qwen2.5-VL-32B} by \textbf{2.70}$\times$ and \texttt{LLaVA-OneVision-72B} by \textbf{2.94}$\times$.
    Notably, it boosts the mean accepted length and speedup ratio by a remarkable \textbf{136\%} and \textbf{35\%} compared to SOTA training-free SD for Video-LLMs, validating our visual-centric design.
\end{itemize}

\section{Preliminary Study}
In this section, we investigate the decoding dynamics of Video-LLMs to motivate our proposed approach. We begin by presenting empirical evidence in~\Cref{sec:empirical_sparsity}, which reveals the \textbf{dual nature} of visual semantic dependence: it is \textbf{sparse in distribution yet decisive for semantic grounding}. Building on this observation, we provide a rigorous theoretical analysis in~\Cref{sec:theoretical_analysis} to \textbf{quantify the theoretical acceleration potential} of exploiting this sparsity, establishing the mathematical foundation for breaking the bottleneck of strictly SD.

\begin{figure*}[t]
  \centering
  \includegraphics[width=\textwidth]{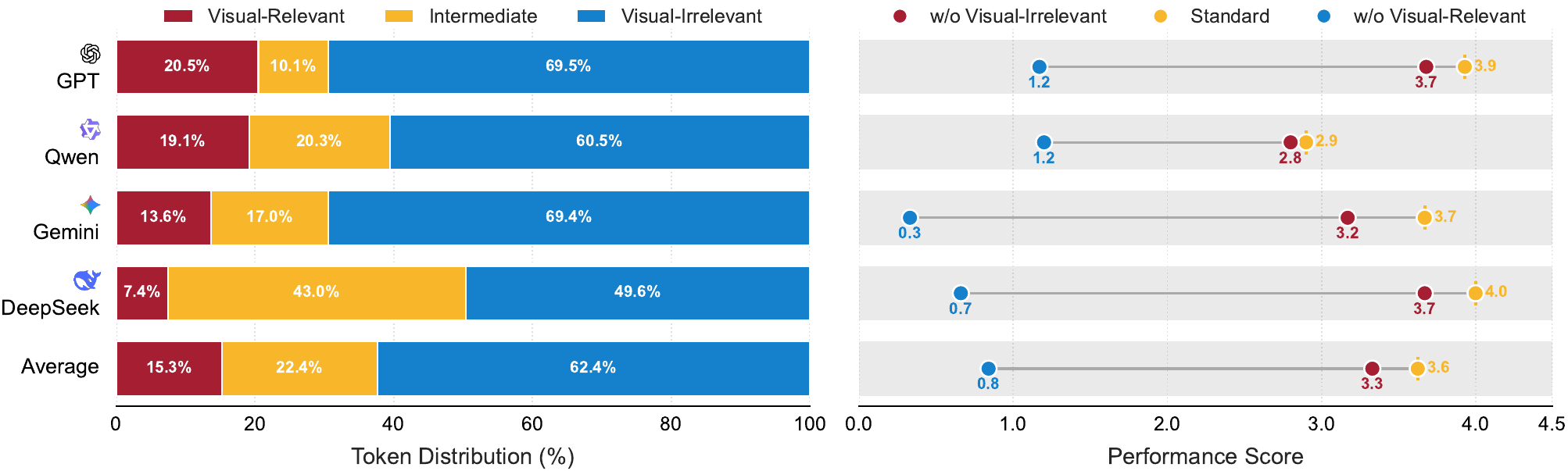}
  \caption{\textbf{Left:} (a) The distribution of visual-relevant and visual-irrelevant tokens.
  \textbf{Right:} (b) LLMs evaluation on Video Detail Caption benchmark. Visual-relevant tokens dominate the output quality.
  }
  \label{fig:observation}
\end{figure*}


\subsection{Sparse yet Critical Visual Semantics}
\label{sec:empirical_sparsity}

To investigate the underlying mechanism of visual understanding in Video-LLMs, we conducted an \textbf{oracle analysis}. We term this an ``oracle'' study because it relies on retrospective classification by a superior model (e.g.,~\texttt{GPT-4o}) using full context information that is theoretically inaccessible during real-time decoding. We utilized four advanced production-grade models (\texttt{GPT-4o}, \texttt{Gemini-3-Pro}, \texttt{DeepSeek-V3}, and \texttt{Qwen3-Max}) on the Video Detail Caption (VDC)~\cite{DBLP:conf/iclr/ChaiSDMMBHXM25} benchmark to address two fundamental questions regarding token distribution and importance.

\noindent \textbf{Question 1:} \textit{Are visual semantics uniformly distributed across the generated sequence or are they concentrated in specific units?}

To address this, we collect generation responses of \texttt{Qwen2.5-VL-32B} on VDC and prompt the four production-grade models to categorize each generated token based on the complete sentence context. We established a taxonomy where 
``Visual-Relevant'' 
refers to tokens directly corresponding to visible scene elements, 
``Visual-Irrelevant'' 
denotes function or grammatical words, abstract terms, conjunctions, and 
``Intermediate'' 
covers other terms\footnote{
\Cref{fig:prompt_template} in~\Cref{app:oracle} shows the prompt template.}.
As shown in~\Cref{fig:observation}, the distribution of these categories reveals a consistent pattern:

\noindent \textbf{Observation 1 (Sparsity):} \textit{Only a small fraction of generated tokens are directly tied to visual elements, while the vast majority consist of functional, grammatical, or abstract fillers that are either completely irrelevant or only loosely associated with the visual content.}

Having established the distributional sparsity of visual tokens, we proceed to examine their contribution to the overall generation quality.

\noindent \textbf{Question 2:} \textit{Are all decoded tokens equally critical for visual understanding, or does their impact on quality vary significantly?}

We investigate this by performing a targeted ablation across all four models, selectively pruning the top-50 most versus least visually relevant tokens from the normalized outputs. The subsequent evaluation (shown in~\Cref{fig:observation} (b)) demonstrates a sharp contrast: removing visually relevant tokens leads to a substantial collapse in generation quality, whereas removing high-frequency irrelevant tokens results in no noticeable performance drop. This leads to our second observation:

\noindent \textbf{Observation 2 (Criticality):} \textit{The semantic integrity of the generation relies heavily on the presence of visually relevant tokens; functional and irrelevant tokens, despite their high frequency, contribute negligibly to the core visual understanding.}


We distill these empirical observations into the following hypothesis:
\begin{tcolorbox}[colback=gray!10, colframe=gray!20, arc=2mm, boxrule=0.5pt, left=1mm, right=1mm, top=1mm, bottom=1mm]
\textbf{Visual Anchor Sparsity Hypothesis:} Generation quality is governed by a sparse set of ``anchor'' tokens, while dense irrelevant ``filler'' tokens carry minimal information density.
\end{tcolorbox}
\noindent This insight motivates our verification strategy: enforcing strict verification on anchors while relaxing verification for the remaining fillers aggressively~\footnote{See \Cref{fig:case_study_oracle} in \Cref{app:oracle} for a detailed case analysis.}.

\subsection{Theoretical Benefits of Loose Verification}
\label{sec:theoretical_analysis}

We first formalize the speculative decoding (SD) process. Given a target model (verifier) $\mathcal{M}_{t}$ and a draft model (drafter) $\mathcal{M}_{d}$, in each decoding step $s$, the draft model $\mathcal{M}_{d}$ autoregressively produces $K$ draft tokens $\{\hat{y}_k^s\}_{k=1}^K$, which are then verified by the target model in parallel:
\begin{equation}
\{{y}_k^s\}_{k=1}^K = \mathcal{M}_{t}([y^0_: \cdots y^{s-1}_:, \{\hat{y}_k^s\}_{k=1}^K]),
\end{equation}
where $\{{y}_k^s\}_{k=1}^K$ denotes the verified sequence and $y^{0}_: \cdots y^{s-1}_:$ represent the context history. Following~\citet{li2025looselysd}, we define the verification strategy $\triangledown$ to determine whether the draft sequence matches the verified sequence at each position ($1$ if matched, else $0$). The accepted length $\tau_s^{\triangledown}$ (a random variable) for the current step $s$ is defined using the longest common prefix logic:
\begin{equation}
\tau_s^{\triangledown} \coloneqq \sum\nolimits_{k=1}^{K} \left( \prod\nolimits_{j=1}^{k} \triangledown(\hat{y}_j^s, y_j^s) \right).
\end{equation}
For notational brevity, we assume the decoding process is statistically stationary across steps, rendering the expected accepted length invariant to $s$ \footnote{
This simplification does not alter our findings.
}. Thus, we drop the index and denote the expectation as $\mathbb{E}[\tau^{\triangledown}]$. Let $T_{t}$ and $T_{d}$ be the latency of $\mathcal{M}_{t}$ and $\mathcal{M}_{d}$ per step, respectively. The overall speedup ratio is defined as:
\begin{equation}
\label{eq:speedup}
\text{speedup} \coloneqq \frac{\mathbb{E}[\tau^{\triangledown}] \cdot T_{t}}{K \cdot T_{d} + T_{t}^{K}}.
\end{equation}
Here $T_t^K$ denotes the latency for $\mathcal{M}_{t}$ to verify $K$ draft tokens in parallel. Based on~\Cref{eq:speedup}, we analyze the behavior of $\mathbb{E}[\tau^{\triangledown}]$ under different verification strategies. Specifically, we \textbf{contrast the theoretical limits} of standard strict verification ($\mathbb{E}[\tau_{\text{strict}}^{\triangledown}]$) \textbf{against the scaling capabilities} of loose verification, quantified by the ratio $\frac{\mathbb{E}[\tau^{\triangledown}_{\text{loose}}]}{\mathbb{E}[\tau^{\triangledown}_{\text{strict}}]}$.

\begin{proposition}[Acceptance Bound of Strict Verification]
\label{prop:strict_bound}
Let $\alpha$ denote the draft alignment accuracy and $\epsilon \coloneqq 1 - \alpha$ be the failure rate. The expected accepted length for strict verification is theoretically bounded by the inverse of the failure rate (proof in~\Cref{app:proof_proposition_2.2}).\footnote{Based on the standard i.i.d. assumption in~\Cref{app:assumption}.}
\end{proposition}
\begin{equation}
\label{eq:bound_of_strict}
\mathbb{E}[\tau^{\triangledown}_{\text{strict}}] = \sum\nolimits_{k=1}^{K} \alpha^k \approx \frac{\alpha}{1-\alpha} < \frac{1}{\epsilon}.
\end{equation}
The upper bound established in Proposition~\ref{prop:strict_bound} exposes an intrinsic \textbf{efficiency bottleneck}: the speedup is shackled by the draft model's raw failure rate $\epsilon$. 
Existing strictly SD methods strive to improve the mean accepted length $\mathbb{E}[\tau^{\triangledown}]$ \textit{indirectly}, either by enhancing the alignment of the draft model through curated training or involving heavy draft tree structures. In contrast, loosely SD approaches can \textbf{directly and linearly scale up} $\mathbb{E}[\tau^{\triangledown}]$ without increasing any draft overhead. We formalize this capability in the following theorem.

\begin{theorem}[Scaling Law under Loose Verification]
\label{thm:sparsity_scaling}
Let the token sequence be partitioned into sparse yet critical visual-relevant tokens $\mathcal{V}$ and dense yet redundant visual-irrelevant tokens $\mathcal{V}'$ with a visual density $\rho \coloneqq |\mathcal{V}|/K$, where $\rho \ll 1$. 
By \textbf{exploiting this sparsity}, the proposed loose verification mechanism achieves a scaling ratio $\gamma$ inversely proportional to the visual density (proof in~\Cref{app:proof_theorem_2.3}):
\end{theorem}
\begin{equation}
\frac{\mathbb{E}[\tau^{\triangledown}_{\text{loose}}]}{\mathbb{E}[\tau^{\triangledown}_{\text{strict}}]} \approx \frac{1}{\rho}.
\end{equation}

\begin{remark}[Diluting the Failure Rate]
Theorem~\ref{thm:sparsity_scaling} offers a rigorous guarantee: our mechanism mathematically \textbf{dilutes the raw failure rate $\epsilon$ into a significantly lower effective failure rate $\rho\epsilon$}. By effectively ``bypassing'' the strict geometric decay on visual-irrelevant tokens, we enable SD to surpass the theoretical ceiling imposed by the draft model alone. 
Thus, such a loose verification exhibits high tolerance to draft errors, significantly expanding the feasible region for acceleration even when the base alignment $\alpha$ is suboptimal.
\end{remark}

Driven by the empirical observations and theoretical insights, we present \method, a novel loosely speculative decoding for Video-LLMs. \method operationalizes the concept of loose verification by leveraging visual semantics, thereby translating the aforementioned theoretical potential into practical wall-clock speedup.
\section{\method}

\begin{figure}[t]
  \centering
  \includegraphics[width=0.48\textwidth]{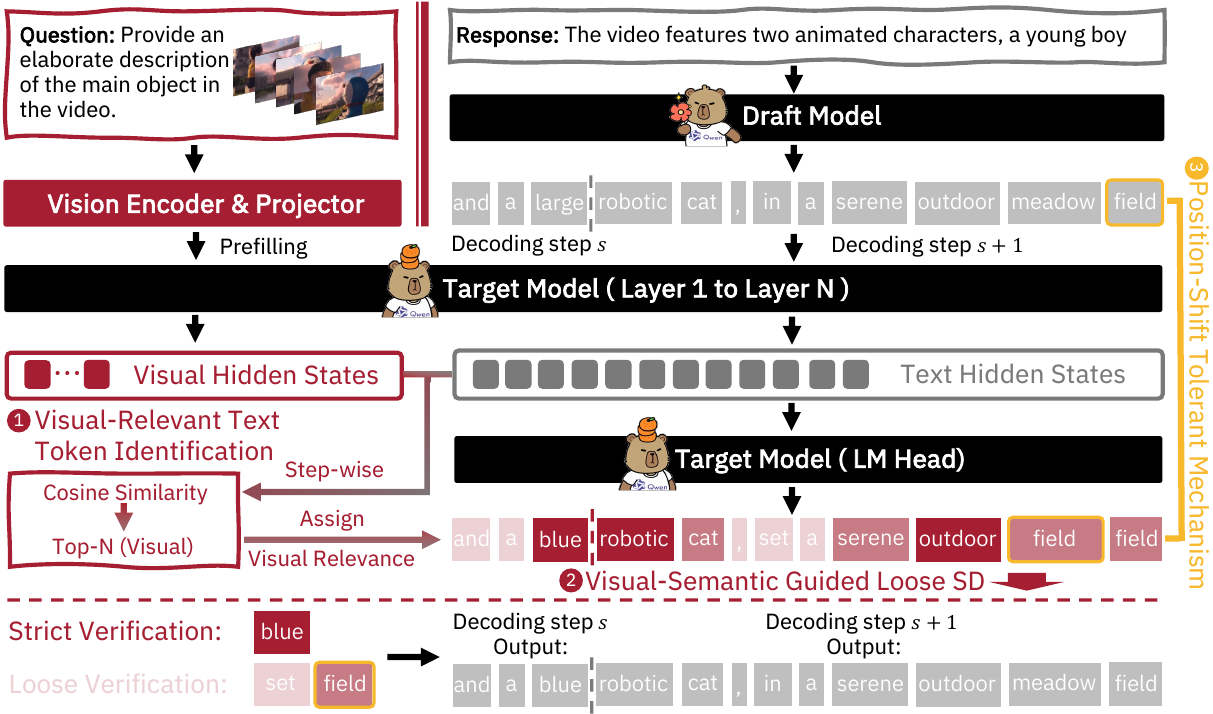}
  \caption{Overview of \method.}
  \label{fig:method}
\end{figure}

\begin{figure*}[t]
\centering
  \includegraphics[width=\textwidth]{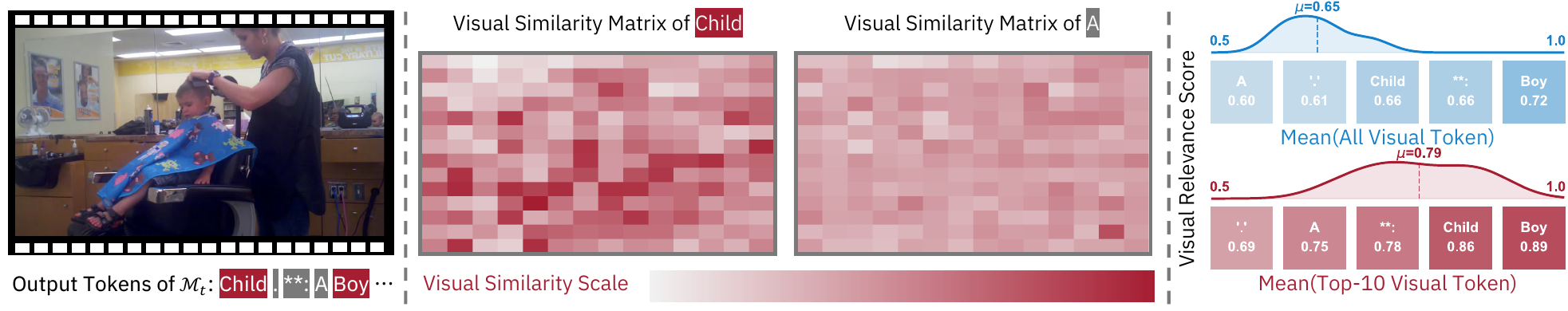}
  \caption{\textbf{Left:} (a) Decoded tokens from the target model $\mathcal{M}_t$. \textbf{Middle:} (b) Visualization of the visual similarity matrix for the decoded tokens. \textbf{Right:} (c) Distribution of visual relevance score w/o and w/ Top-N (N=10) selection.}
  \label{fig:method_combine}
\end{figure*}

%
As depicted in~\Cref{fig:method}, \method first identifies visual-relevant tokens at each decoding step (\Cref{section:method_1}), then strictly verifies relevant ones while loosely verifying irrelevant ones (\Cref{section:method_2}). Additionally, we design a position shift-tolerant mechanism to accept mismatched tokens appearing in a near span of the draft sequence (\Cref{section:method_3}).

\subsection{Visual-Relevant Text Token Identification }
\label{section:method_1}

In~\Cref{sec:empirical_sparsity}, we empirically demonstrate that visual understanding tasks exhibit a pronounced disparity in sensitivity between visual-relevant and visual-irrelevant tokens in the output. However, to exploit this discrepancy for accelerating inference, we require a metric that can quantify visual relevance \textit{during} the decoding stage. 

In prior work~\cite{DBLP:conf/icml/0020FMZ0CGONKZ25}, the cosine similarity between vision embeddings and text embeddings is used to measure their relevance. Since the vision embeddings are projected into the same textual embedding space via the vision projector, the distance between these two representations can effectively reflect their association in the semantic space within the LLM backbone.
Leveraging such cross-modal similarity, prior work~\cite{DBLP:conf/icml/0020FMZ0CGONKZ25} focuses on text-guided visual token pruning, whereas \method takes the opposite perspective: we novelly apply a vision-guided approach to facilitate text generation.

Specifically, we decompose $\mathcal{M}_{t}$ into the final-layer representation function $\tilde{\mathcal{M}_{t}}(\cdot)$ and the language modeling head $\mathrm{lm\_head}(\cdot)$ as follows:
\begin{equation}
\label{eq:model_layer}
\mathcal{M}_{t}(\mathbf{\cdot}) \rightarrow \mathrm{lm\_head}\!\big(\tilde{\mathcal{M}_{t}}(\mathbf{\cdot})\big),
\end{equation}
where $\tilde{\mathcal{M}_{t}}(\mathbf{\cdot})$ denotes the output hidden states of the last layer of $\mathcal{M}_{t}$, before the language model head. 
We use $\tilde{\mathcal{M}}_{t}(\cdot)$ for subsequent computations because (i) the Transformer’s final-layer hidden states provide a strong feature representation, consistent with the design choice in EAGLE~\cite{DBLP:conf/icml/LiW0024}, and (ii) the sequence produced by $\mathcal{M}_{t}$ determines the final output, rather than $\mathcal{M}_{d}$.



Then, for a given input video and prompt, let $E_{V} \in \mathbb{R}^{l_v \times d}$ be the input embedding after the vision encoder and projector, where $l_v$ is the visual sequence length and $d$ is the hidden dimension of the target model $\mathcal{M}_{t}$.
At each decoding step $s$, the draft model $\mathcal{M}_{d}$ generates $K$ draft tokens, yielding the corresponding draft text embedding $E_{D}^{s} \in \mathbb{R}^{K \times d}$.
We extract the hidden states $H_V$ (during prefilling) and $H_D^s$ (during decoding) as:
\begin{equation}
\label{eq:hidden_states}
\begin{aligned}
H_V &= \tilde{\mathcal{M}_{t}}(E_{V}) \in \mathbb{R}^{l_v \times d}, \\
H_D^s &= \tilde{\mathcal{M}_{t}}(E_{D}^s) \in \mathbb{R}^{K \times d}.
\end{aligned}
\end{equation}

In each decoding step, $H_V$ and $H_D^s$ are used for similarity computation. Normally, visual cues consist of both salient subjects and abundant background elements. 
As visualized in~\Cref{fig:method_combine}~(a,b), we find that visual-relevant tokens associated with the main subject (e.g., ``Child'') exhibit a much sharper similarity profile: they align strongly with the salient subject cues while showing low similarity to the background. In contrast, less relevant tokens (e.g., ``A'') tend to have a more uniform similarity distribution across all background cues.
If the full $H_V$ containing all video cues is used as the reference, the similarity of visual-relevant tokens can be diluted by abundant background cues, making them difficult to distinguish from visual-irrelevant ones.
Therefore, we compute the cosine similarity between the text hidden states and the visual hidden states, and take the average of the Top-N largest similarities along the visual dimension as the final \textbf{visual relevance} metric $\mathcal{A}_s$ for each text token in the decoding step $s$, denoted as:
\begin{equation}
\begin{aligned}
\label{eq:topk_mean}
&\mathcal{A}_s
=\frac{1}{N}\sum\nolimits_{n=1}^{N}\mathrm{TopN}\big(\cos(H_D^s, H_V),\,N\big), \\ &\text{where} \ 
\cos(H_D^s, H_V)
= \frac{ H_D^{s} \cdot H_V}{\lVert H_D^s\rVert\,\lVert H_V\rVert}.
\end{aligned}
\end{equation}
Here, $N$ is a hyper-parameter controlling the number of critical visual tokens that provide guidance. This operation effectively improves the discriminability of the visual relevance (see~\Cref{fig:method_combine}~(c)), with minimal computation cost (see~\Cref{fig:computation_breakdown}).

\subsection{Visual-Semantic Guided Loosely SD}
\label{section:method_2}
Given the visual similarity of text tokens in each decoding step $\mathcal{A}_s$, we leverage it as guidance to perform loosely speculative decoding.
Specifically, we adopt strict verification for the tokens with high visual relevance, while allowing the acceptance of visual-irrelevant tokens. 
Let $\lambda \in [0,1]$ be a loose parameter to control the percentage of visual-irrelevant tokens that are loosened. Then in each decoding step, we consider the $\lambda K$ tokens with the \textit{lowest} visual relevance scores as the practical instantiation of $\mathcal{V}'$, aligning with the redundancy assumption in~\Cref{thm:sparsity_scaling}.
Then the index set of $\mathcal{V}'$ can be denoted as:
\begin{equation}
\label{eq:sort_relevance}
\mathcal{I}_s'=\operatorname{argsort}_{\uparrow}
(\mathcal{A}_s)_{0:\lambda K},
\end{equation}
Finally, we define the verification strategy as $\triangledown_\textsc{LVSpec}$ to decide whether the draft tokens match the target model's verified tokens at each position:
\begin{equation}
\label{eq:verification_strategy_definition}
\triangledown_\textsc{LVSpec}(y_i^s,\hat{y}_i^s) = 
\begin{cases}
1, &  \text{if } y_i^s \text{ matches }\hat{y}_i^s, \\
1, & \text{else if } i \in \mathcal{I}_s', \\
0, & \text{otherwise}.
\end{cases}
\end{equation}
By choosing a larger $\lambda$, we can apply a more aggressive relaxation to pursue higher speedups. Conversely, for tasks that require a more conservative behavior, we can use a smaller $\lambda$, eventually recovering the standard exact-match scheme as $\lambda$ decreases.
This mechanism allows the level of relaxation to be adaptively tuned according to task characteristics and model behaviors.

Notably, the verification above is based on the local visual relevance among text tokens at each decoding step.
This design is chosen because hidden state values of nearby text tokens are more directly comparable and are less susceptible to absolute drifts induced by the growing generation context.

\subsection{Position Shift-Tolerant Mechanism}
\begin{figure}[!ht]
  \centering
  \includegraphics[width=0.48\textwidth]{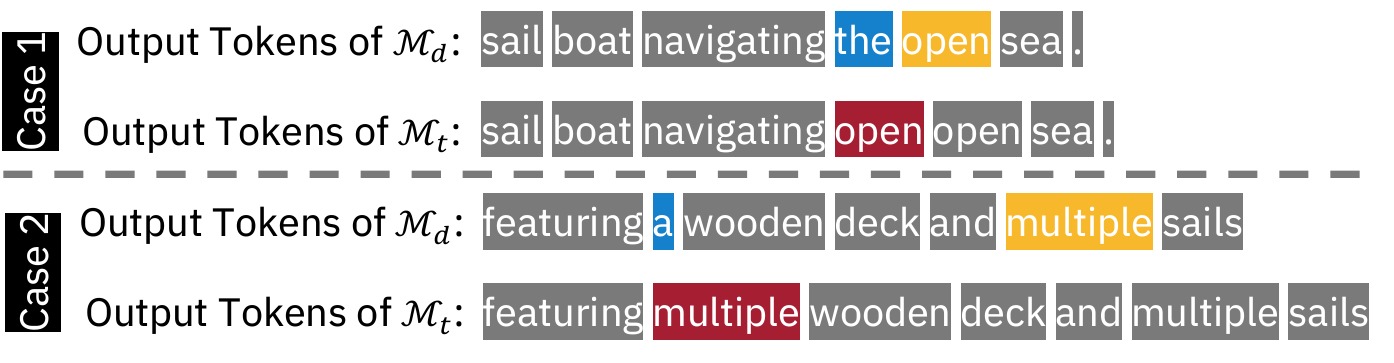}
  \caption{Insight cases of \textsc{PST}. Key background colors denote \colorboxlabel{hanblue}{Mismatched Draft Token}, \colorboxlabel{hanyellow}{Shifted Draft Token}, and \colorboxlabel{hanred}{Verified Target Token}, respectively.}
  \label{fig:pst_insight}
\end{figure}
\label{section:method_3}

As illustrated in~\Cref{fig:pst_insight}, we observe that in video understanding outputs, occasional mismatches arise due to positional shifts. Specifically, such shifts are typically caused by the draft model introducing additional phrases (token ``the'' in \textit{case 1}), or by differences in the description order between the draft and target models (token ``multiple'' in \textit{case 2}). 
We argue that in such cases, the target model’s intended output is covered by the draft model’s output distribution. In~\Cref{fig:pst_insight}, the draft output is actually aligned with that of the target model and relaxing the currently mismatched token may not lead to a loss of visual information in subsequent generations.
Based on this finding, we further incorporate a position shift-tolerant mechanism (\textsc{PST}), a bonus component for such cases.
If the currently mismatched verified token appears within a nearby span of the draft sequence, we treat the mismatch as position shift-induced and accept it.
Specifically, we set the draft length $K$ at each decoding step as the tolerable shift range, and the verification strategy $\triangledown_\textsc{LVSpec}$ is reformulated as:
\begin{equation}
\label{eq:verification_strategy_definition2}
\triangledown_\textsc{LVSpec}(y_i^s,\hat{y}_i^s) = 
\begin{cases}
1, &  \text{if } y_i^s \text{ matches }\hat{y}_i^s, \\
1, & \text{else if } i \in \mathcal{I}_s', \\
1, & \text{else if } y_i^s \text{ is in } \{\hat{y}_k^s\}_{k=1}^K, \\
0, & \text{otherwise}.
\end{cases}
\end{equation}
Finally, the accepted length is defined using the longest common prefix logic.
Note that \textsc{PST} serves only as a supplementary strategy after visual-semantic guidance is applied,
selectively enabled in minority cases depending on the output pattern\footnote{Potential improvement of \method is in~\Cref{app:potential_improvement}.}.



\section{Experiment}

\begin{table*}[t]
  \centering
  \scriptsize
  \setlength{\tabcolsep}{10pt}
  \renewcommand{\arraystretch}{0.5} 
  \begin{tabular}{l|l|ccc|ccc|ccc}
    \toprule
    \multicolumn{1}{l|}{\multirow{2}{*}{\textbf{Model}}} &
    \multicolumn{1}{l|}{\multirow{2}{*}{\textbf{Method}}} &
    \multicolumn{3}{c|}{\textbf{Video Detail Caption (VDC)}} &
    \multicolumn{3}{c|}{\textbf{Video Detail Description (VDD)}} &
    \multicolumn{3}{c}{\textbf{MovieChat}} \\
    & &
    $\tau$ & Token/s & Speedup &
    $\tau$ & Token/s & Speedup &
    $\tau$ & Token/s & Speedup \\
    \midrule

    \multirow{6}{*}{\rotatebox{90}{\makecell[c]{\ttfamily Std.-SD \\ \ttfamily Qwen2.5-VL}}}
    & \gray{Autoregressive}
      & \gray{--}   & \gray{7.95}  & \gray{1.00$\times$}
      & \gray{--}   & \gray{10.07} & \gray{1.00$\times$}
      & \gray{--}   & \gray{9.81}  & \gray{1.00$\times$} \\
    & \textsc{Naive SD}
      & 3.28 & 11.17 & 1.41$\times$
      & 3.36 & 14.44 & 1.42$\times$
      & 3.19 & 13.38 & 1.36$\times$ \\
    & \textsc{SpecVLM}
      & 3.29 & 15.90 & 2.00$\times$
      & 3.31 & 17.96 & 1.76$\times$
      & 3.19 & 17.05 & 1.74$\times$ \\
    &$\textsc{FLy}^{\ominus}$
      & 6.81 & 18.50 & 2.32$\times$
      & 7.20 & 22.14 & 2.20$\times$
      & 6.74 & 20.76 & 2.12$\times$ \\
    & \textsc{FLy}
      & 4.37 & 13.60 & 1.71$\times$
      & 4.34 & 14.51 & 1.44$\times$
      & 3.86 & 13.01 & 1.33$\times$ \\
    & \hanredcell{\method (Ours)}
      & \hanredcell{\textbf{7.76}} & \hanredcell{\textbf{21.47}} & \hanredcell{\textbf{2.70$\times$}}
      & \hanredcell{\textbf{7.68}} & \hanredcell{\textbf{23.58}} & \hanredcell{\textbf{2.34$\times$}}
      & \hanredcell{\textbf{7.74}} & \hanredcell{\textbf{21.49}} & \hanredcell{\textbf{2.19$\times$}} \\
    \midrule

    \multirow{6}{*}{\rotatebox{90}{\makecell[c]{\ttfamily Self-SD \\ \ttfamily Qwen2.5-VL}}}
    & \gray{Autoregressive}
      & \gray{--}   & \gray{20.60} & \gray{1.00$\times$}
      & \gray{--}   & \gray{24.75} & \gray{1.00$\times$}
      & \gray{--}   & \gray{25.51} & \gray{1.00$\times$} \\
    & \textsc{Naive SD}
      & 4.43 & 18.21 & 0.88$\times$
      & 4.46 & 23.22 & 0.94$\times$
      & 4.46 & 21.86 & 0.86$\times$ \\
    & \textsc{SpecVLM}
      & 3.85 & 31.45 & 1.53$\times$
      & 3.89 & 35.00 & 1.41$\times$
      & 3.66 & 32.20 & 1.26$\times$ \\
    &$\textsc{FLy}^{\ominus}$
      & 6.86 & 32.24 & 1.56$\times$
      & 7.17 & 33.58 & 1.36$\times$
      & 5.98 & 28.23 & 1.11$\times$ \\
    & \textsc{FLy}
      & 6.53 & 28.16 & 1.37$\times$
      & 6.35 & 30.30 & 1.22$\times$
      & 5.44 & 25.82 & 1.01$\times$ \\
    & \hanredcell{\method (Ours)}
      & \hanredcell{\textbf{8.36}} & \hanredcell{\textbf{36.49}} & \hanredcell{\textbf{1.77$\times$}}
      & \hanredcell{\textbf{8.64}} & \hanredcell{\textbf{39.32}} & \hanredcell{\textbf{1.59$\times$}}
      & \hanredcell{\textbf{7.99}} & \hanredcell{\textbf{34.04}} & \hanredcell{\textbf{1.33$\times$}} \\
    \midrule

    \multirow{6}{*}{\rotatebox{90}{\makecell[c]{\ttfamily Std.-SD \\ \ttfamily LLaVA-OV}}}
    & \gray{Autoregressive}
      & \gray{--}   & \gray{6.62}  & \gray{1.00$\times$}
      & \gray{--}   & \gray{6.00}  & \gray{1.00$\times$}
      & \gray{--}   & \gray{4.41}  & \gray{1.00$\times$} \\
    & \textsc{Naive SD}
      & 3.40 & 12.50 & 1.89$\times$
      & 3.72 & 11.74 & 1.96$\times$
      & 3.88 & 5.36  & 1.22$\times$ \\
    & \textsc{SpecVLM}
      & 3.37 & 15.75 & 2.38$\times$
      & 3.53 & 14.94 & 2.49$\times$
      & 3.77 & 8.45  & 1.92$\times$ \\
    &$\textsc{FLy}^{\ominus}$
      & 4.83 & 13.82 & 2.09$\times$
      & 4.93 & 11.90 & 1.98$\times$
      & 5.55 & 9.23  & 2.09$\times$ \\
    & \textsc{FLy}
      & 4.12 & 12.17 & 1.84$\times$
      & 4.58 & 11.50 & 1.92$\times$
      & 5.01 & 5.72  & 1.30$\times$ \\
    & \hanredcell{\method (Ours)}
      & \hanredcell{\textbf{7.34}} & \hanredcell{\textbf{19.44}} & \hanredcell{\textbf{2.94$\times$}}
      & \hanredcell{\textbf{7.59}} & \hanredcell{\textbf{17.45}} & \hanredcell{\textbf{2.91$\times$}}
      & \hanredcell{\textbf{8.10}} & \hanredcell{\textbf{12.16}} & \hanredcell{\textbf{2.75$\times$}} \\
    \midrule

    \multirow{6}{*}{\rotatebox{90}{\makecell[c]{\ttfamily Self-SD \\ \ttfamily LLaVA-OV}}}
    & \gray{Autoregressive}
      & \gray{--}   & \gray{26.45} & \gray{1.00$\times$}
      & \gray{--}   & \gray{17.66} & \gray{1.00$\times$}
      & \gray{--}   & \gray{17.23} & \gray{1.00$\times$} \\
    & \textsc{Naive SD}
      & 4.58 & 23.83 & 0.90$\times$
      & 4.59 & 15.74 & 0.89$\times$
      & 4.79 & 11.16 & 0.65$\times$ \\
    & \textsc{SpecVLM}
      & 3.91 & 35.46 & 1.34$\times$
      & 3.95 & 29.08 & 1.65$\times$
      & 4.17 & \textbf{21.19} & \textbf{1.23$\times$} \\
    &$\textsc{FLy}^{\ominus}$
      & 6.83 & 31.72 & 1.20$\times$
      & 6.82 & 27.01 & 1.45$\times$
      & 8.06 & 17.94 & 1.04$\times$ \\
    & \textsc{FLy}
      & 6.23 & 29.39 & 1.11$\times$
      & 6.20 & 25.57 & 1.45$\times$
      & 7.65 & 17.29 & 1.00$\times$ \\
    & \hanredcell{\method (Ours)}
      & \hanredcell{\textbf{8.46}} & \hanredcell{\textbf{36.40}} & \hanredcell{\textbf{1.38$\times$}}
      & \hanredcell{\textbf{8.87}} & \hanredcell{\textbf{34.96}} & \hanredcell{\textbf{1.98$\times$}}
      & \hanredcell{\textbf{8.77}} & \hanredcell{20.28} & \hanredcell{1.18$\times$} \\
    \bottomrule
  \end{tabular}
  \caption{
  Speedup ratios and mean accepted length $\tau$ on video captioning and video question answering tasks.
  }
  \label{tab:speedup}
\end{table*}

\begin{figure*}[t]
  \includegraphics[width=\textwidth]{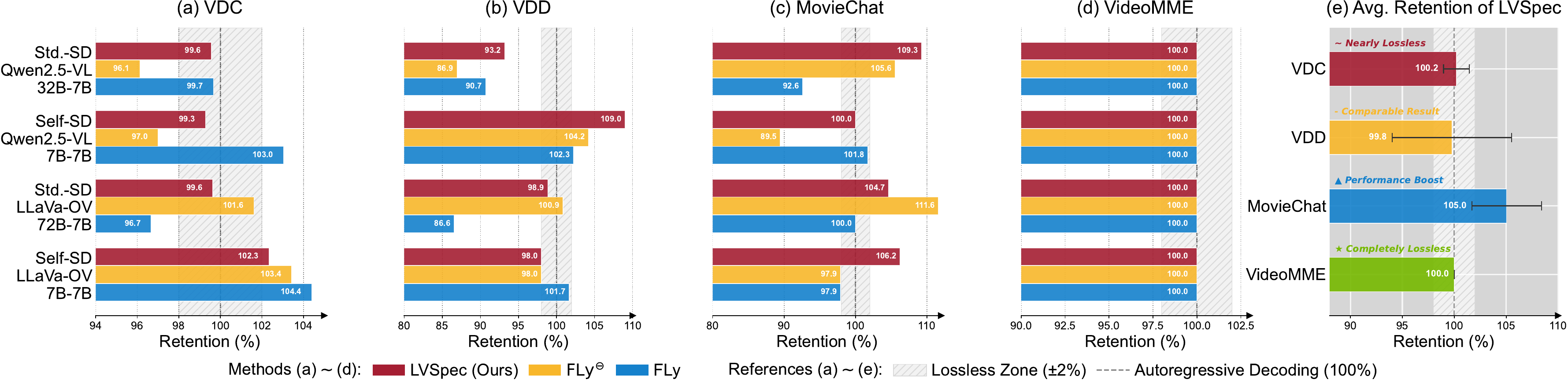}
  \caption{
    Performance retention of loosely SD methods.~\Cref{tab:accuracy} in~\Cref{app:results} shows the raw performance data.
  }
  \label{fig:accuracy}
\end{figure*}

\paragraph{Tasks and Benchmarks.}
We conduct experiments on four video captioning and video question answering tasks, including Video Detail Caption (VDC)~\cite{DBLP:conf/iclr/ChaiSDMMBHXM25}, Video Detail Description (VDD)~\cite{vdc2024}, MovieChat~\cite{DBLP:conf/cvpr/SongCWZZWCG0ZLH24}, and Video-MME~\cite{DBLP:conf/cvpr/FuDLLRZWZSZCLLZ25}, using \texttt{LMMs-Eval}~\cite{lmms_eval2024} for evaluation.
Benchmark details are provided in~\Cref{app:benchmark_detail}.

\paragraph{Baselines.}
We compare our method with both training-free lossless SD methods and loosely SD methods. 
(i) \textbf{Lossless SD methods}. \ding{172} \textsc{Naive SD} directly uses an existing model as the draft model. \ding{173} \textsc{SpecVLM}~\cite{ji2025specvlm} applies uniform video token pruning to the draft model, reserving 10\% of the visual tokens to make it compact. 
(ii) \textbf{Loosely SD methods}. 
\ding{174} \textsc{FLy}~\cite{li2025looselysd} consists of an entropy gate and a deferred window: a mismatched token is accepted when its entropy exceeds a threshold and all tokens within the subsequent window match. 
\ding{175} $\textsc{FLy}^{\ominus}$ uses only the entropy gate component of \textsc{FLy}, serving as a more loosened variant.
Both \textsc{Naive SD} and \textsc{SpecVLM} adopt the best-performing draft tree configuration in the original paper, consistent with EAGLE~\cite{DBLP:conf/icml/LiW0024}. For loosely SD methods $\textsc{FLy}^{\ominus}$, \textsc{FLy}, and \method, we use the same number of draft tokens ($K=10$), forming a chain-like structure that also matches the design in \textsc{FLy}.

\paragraph{Target and Draft Models.}
We choose two representative Video-LLMs families: \texttt{Qwen2.5-VL}~\cite{DBLP:journals/corr/abs-2502-13923} and \texttt{LLaVa-OneVision}~\cite{DBLP:journals/tmlr/0080ZGZ00ZZL0L25} as our target and draft models. 
Aligning with \textsc{SpecVLM}, we evaluate two speculative decoding settings. (i) \textbf{Standard-SD (Std.-SD)}: using a smaller Video-LLM from the same model family.
``\texttt{Std.-SD-Qwen2.5-VL}'' and ``\texttt{Std.-SD-LLaVA-OV}'' denote settings that use \texttt{Qwen2.5-VL-32B} and \texttt{LLaVA-OneVision-72B} as the target models, and \texttt{Qwen2.5-VL-7B} and \texttt{LLaVA-OneVision-7B} with video token pruning applied as the corresponding draft models, respectively.
(ii) \textbf{Self-SD}: using the original model and full video tokens as verifier, while using the original model with pruned video tokens as drafter.
``\texttt{Self-SD-Qwen2.5-VL}'' and ``\texttt{Self-SD-LLaVA-OV}'' adopt \texttt{Qwen2.5-VL-7B} and \texttt{LLaVA-OneVision-7B}, respectively.
A study of additional draft model sizes is provided in~\Cref{app:draft_model_size}.
We do not introduce additional training for the draft model, as our primary focus is to isolate how the verification strategy affects the mean accepted length $\tau$ and speedup ratio.

\paragraph{Implementation Details.}
For all experiments, we set the loose parameter $\lambda=0.7$ and the number of critical visual tokens $N=10$. 
We set the maximum generation length to 512 tokens and conduct all experiments on two NVIDIA H200 GPUs.
More details of the implementation of \textsc{FLy}~\cite{li2025looselysd} and \textsc{SpecVLM}~\cite{ji2025specvlm} are elaborated in~\Cref{app:baseline_detail}.


\subsection{Main Result}
We evaluate the \textbf{efficiency metric} including speedup ratio and the mean accepted token length $\tau$, and the \textbf{performance metric} including rating score and accuracy on video understanding tasks, as summarized in~\Cref{tab:speedup} and~\Cref{fig:accuracy}.

\textbf{\method shows superior speedup and mean accepted length $\tau$ across various model settings and video understanding tasks}. 
As shown in~\Cref{tab:speedup}, \method achieves the highest speedups of \textbf{$2.70\times$} and \textbf{$2.94\times$} under the \texttt{Std.-SD-Qwen2.5-VL} and \texttt{Std.-SD-LLaVA-OV} settings, respectively. 
The mean accepted length $\tau$ of \method in the \texttt{Std.-SD} setting ranges from 7.34 to 8.10, which is $2.15\times$ that of the lossless SD baseline \textsc{SpecVLM} and $>1.46\times$ that of the loosely SD baseline $\textsc{FLy}^{\ominus}$ and \textsc{FLy}.
For the \texttt{Self-SD-Qwen2.5-VL} and \texttt{Self-SD-LLaVA-OV} settings, \method also achieves the highest mean accepted length $\tau$ (8.87) and the highest speedup ($1.98\times$). 
In addition, the observed improvements of \method with $\lambda=0.7$ closely matches the effect of $\rho$ on the failure rate analyzed in~\Cref{thm:sparsity_scaling} (see~\Cref{app:results} for a detailed explanation).
On tasks that require longer output such as VDC and VDD, \method tends to achieve higher speedups.
On \texttt{Self-SD LLaVA-OV} setting and MovieChat task, \method attains a slightly lower speedup than \textsc{SpecVLM}. 
We attribute this to the already high $\tau$ before relaxation in this setting, which leaves limited room for relaxation, as well as the fact that \textsc{SpecVLM} employs a draft tree with more verification nodes (26) than \method (10).
In~\Cref{app:draft_tree}, we additionally report results of \textsc{SpecVLM} w/o draft trees as well as \method incorporating draft trees.

\textbf{\method achieves nearly lossless performance on video captioning and open-ended question answering tasks, and remains fully lossless on multiple-choice task compared with the target model's original output.}
As shown in~\Cref{fig:accuracy}, \method achieves average accuracy retention ratios of 100.2\%, 99.8\%, 105.0\%, and 100.0\% on VDC, VDD, MovieChat, and Video-MME, respectively. 
In comparison, the retention ratios of the draft model on the four tasks are 94.7\%, 95.2\%, 105.9\%, and 89.3\% on average.
Across all model settings and tasks, \method retains at least 93.2\% of the original accuracy, outperforming \textsc{FLy}$^{\ominus}$ (86.9\%) and \textsc{FLy} (86.6\%), and demonstrating better overall stability.
While preserving performance, \method also achieves higher speedups than these baselines (1.58$\times$ of \textsc{FLy} on VDC), indicating that its loosened criterion is more precise and effective.
Notably, on MovieChat, \method even yields a 5.0\% accuracy improvement. We attribute this to the higher sensitivity of short generation tasks to small perturbations, and that some existing small models can occasionally outperform the large model in certain settings (see~\Cref{tab:accuracy} in~\Cref{app:results}).

\begin{figure}[t]
  \includegraphics[width=\columnwidth]{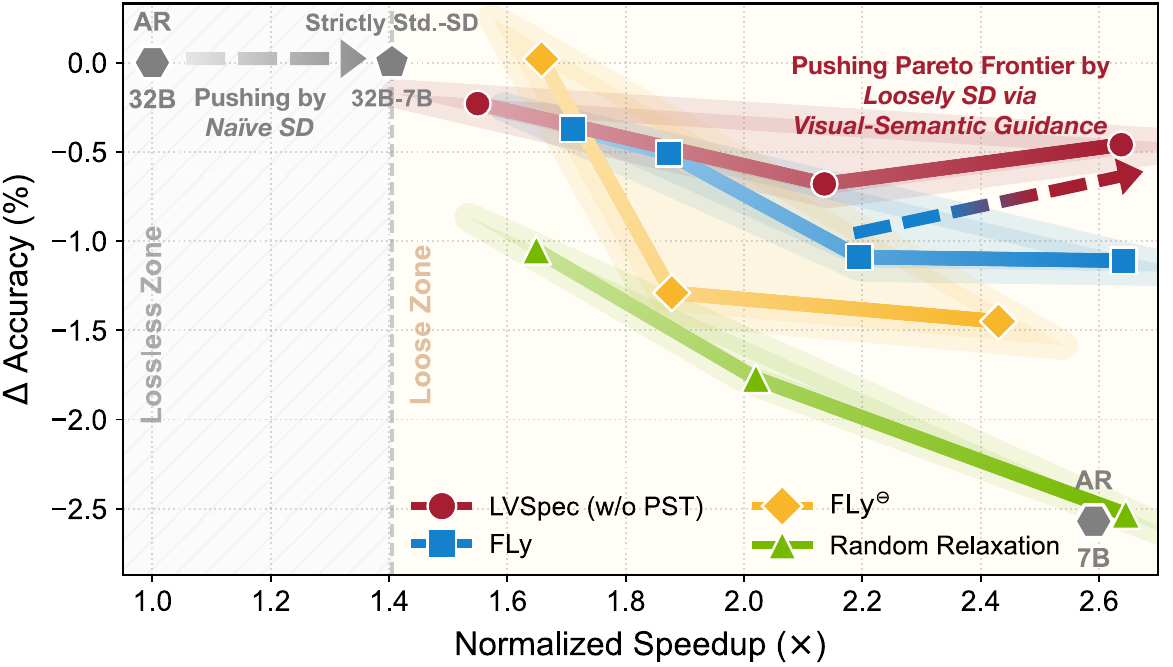}
  \caption{The Pareto frontier of accuracy and speedup for Video-LLM decoding methods under \texttt{Qwen2.5-VL-32B/7B} and VDC task.}
  \label{fig:effective}
\end{figure}


\begin{table}[t]
  \centering
  \scriptsize
  \setlength{\tabcolsep}{11pt}
 \renewcommand{\arraystretch}{1.1} 
  \begin{tabular}{@{}l|cccc@{}}
    \toprule
    Model & \textsc{PST} & $\tau$ & Speedup & Retention (\%) \\
    \midrule
    \multirow{2}{*}{\makecell[l]{\texttt{Std.-SD} \\ \texttt{Qwen2.5-VL}}}
      & $\times$      & 7.27 & 2.54$\times$ & 98.5 \\
      & $\checkmark$  & 7.76 & 2.70$\times$ & 99.6 \\
    \midrule
    \multirow{2}{*}{\makecell[l]{\texttt{Std.-SD} \\ \texttt{LLaVa-OV}}}
      & $\times$      & 6.89 & 2.82$\times$ & 99.1 \\
      & $\checkmark$  & 7.34 & 2.94$\times$ & 99.8 \\
    \bottomrule
  \end{tabular}
  \caption{Ablation study on position shift-tolerant (\textsc{PST}).}
  \label{tab:effect_pst}
\end{table}

\subsection{Ablation Study}
We study the effect of the visual relevance metric and the effect of the \textsc{PST}. All experiments in this section are conducted under \texttt{Qwen2.5-VL-32B/7B} and VDC benchmark, using accuracy as metric.

\paragraph{Effect of Visual-Semantic Guidance.}
\label{sec:ablation_effect_of_visual}
Here, we evaluate the loose criterion using only visual-semantic guidance (w/o \textsc{PST}, $\lambda\in\{0.2, 0.5, 0.7\}$) against baselines including $\textsc{FLy}^{\ominus}$ ($\mathrm{Entropy}\in\{0.1, 0.15, 0.175\}$), \textsc{FLy} ($(\mathrm{Window},\mathrm{Entropy})\in\{(4,0.1), (2,0.1), (1,0.1),(1,0.05)\}$), and random relaxation with sampling probabilities ($\{0.25, 0.5, 0.75\}$).
As shown in~\Cref{fig:effective}, \method effectively pushes the Pareto frontier by leveraging visual relevance to precisely loosen the verification criterion.
Consistent with~\Cref{thm:sparsity_scaling}, \method achieves a superior accuracy-speedup trade-off compared to ``visually blind'' baselines, which fail to maintain effectiveness under higher relaxation levels. While lossless SD ensures quality, \method enables significant acceleration in the loose regime with only marginal quality costs.

\paragraph{Effect of \textsc{PST}.}
For the position shift-tolerant mechanism (\textsc{PST}), as shown in \Cref{tab:effect_pst}, incorporating \textsc{PST} boosts the speedup from $2.54\times$ to $2.70\times$ and $\tau$ from 7.27 to 7.76 for \texttt{Std.-SD-Qwen2.5-VL}, with consistent gains for \texttt{Std.-SD-LLaVA-OV}. 
Notably, \textsc{PST} does not reduce accuracy, indicating that relaxing a small portion of position shift-induced mismatches is benign.

\begin{table}[t]
\centering
\scriptsize
\setlength{\tabcolsep}{10pt}
\renewcommand{\arraystretch}{0.6} 
\begin{tabular}{@{}l|cccc@{}}
\toprule
Hyperparameter & Val  & $\tau$ & Speedup & Retention (\%) \\
\midrule
\multirow{5}{*}{$\lambda$}
  & 0.0 & 3.41 & 1.40$\times$ & 100 \\
  & 0.2 & 4.12 & 1.55$\times$ & 99.3 \\
  & 0.5 & 5.83 & 2.14$\times$ & 97.8 \\
  & 0.7 & 7.27 & 2.54$\times$ & 98.5  \\
  & 0.9 & 8.97 & 3.03$\times$ & 92.8  \\
\midrule
\multirow{3}{*}{$N$}
  & 1   & 7.15 & 2.46$\times$ & 94.6 \\
  & 10 & 7.27 & 2.54$\times$ & 98.5  \\
  & 100 & 7.47 & 2.65$\times$ & 93.4  \\
\bottomrule
\end{tabular}
\caption{Sensitivity study of \method (w/o PST) on loose hyperparameter $\lambda$ and critical visual token $N$.}
\label{tab:hyper-parameter}
\end{table}

\subsection{Sensitivity Study on Hyperparameters}
Here, we investigate the two hyperparameters of \method, as illustrated in~\Cref{tab:hyper-parameter}.
For the loose parameter $\lambda$, we evaluate a range of values in $[0,1]$, spanning verification strategies from conservative to aggressive. For smaller $\lambda$, the speedup achieved by \method increases as $\lambda$ grows, while accuracy remains $>97.8\%$. Under a highly relaxed setting ($\lambda=0.9$), excessive relaxation leaves more visual-relevant tokens without strict verification, leading to an accuracy drop to 92.8\%. Accordingly, we adopt a moderate choice of $\lambda=0.7$, which strikes a favorable balance between speed and accuracy.

For the number of critical visual tokens $N$, we evaluate $N\in\{1,5,10,100\}$. Using too few critical tokens (e.g.,~$N=1$) fails to capture sufficient visual cues, whereas using too many (e.g.,~$N=100$) can partially dilute truly salient information (refer to~\Cref{fig:method_combine}), both leading to an accuracy drop. Thus, we adopt a moderate choice of $N=10$.
\footnote{Experiments on draft model sizes and draft tree variants are elaborated in~\Cref{app:draft_model_size,app:draft_tree}.}

\subsection{What Tokens Are Loosened?}
\label{app:what_tokens_are_loosened}
To improve the interpretability of \method, we present a case study in~\Cref{fig:case_study1}.
In the model output, words closely tied to the salient subjects and actions such as ``engaging'', ``cocktail'', ``man'', and ``bowls'', are strictly verified, preserving the quality of the video-to-text generation. In contrast, conjunctions, punctuation, and function words such as ``depict'', ``the'', ``.'', and ``\#\#'' are relaxed, substantially boosting the mean accepted length $\tau$.
From a broader perspective on vision-to-text tasks, \method differs from prior work that sparsifies information on the \textit{vision} side of the input. Instead, it opens a new direction by inducing sparsity on the \textit{text} side of the output for orthogonal acceleration.

\begin{figure}[!ht]
  \includegraphics[width=0.5\textwidth]{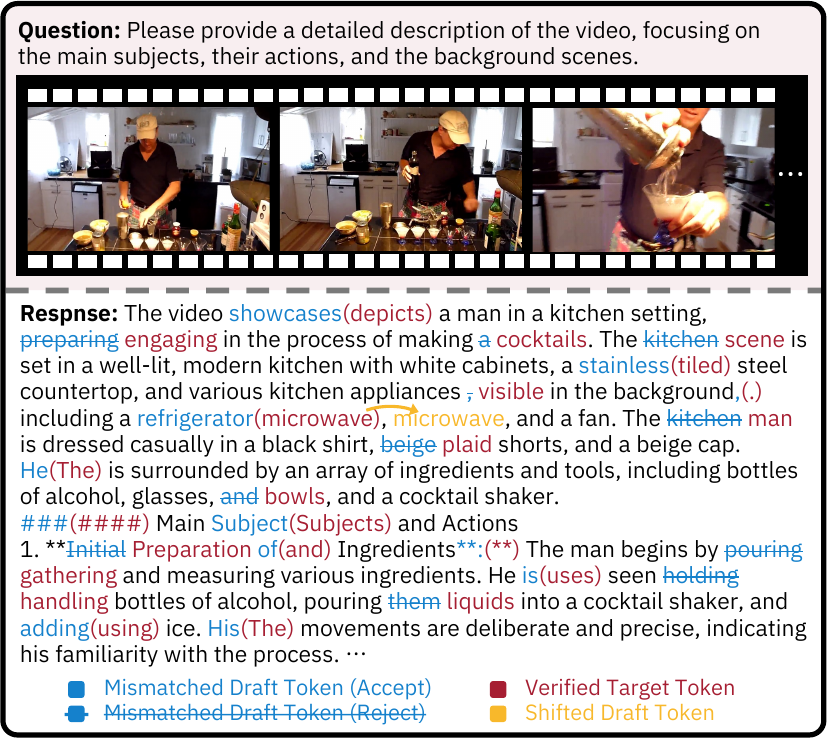}
  \caption{Case study of which tokens are strictly verified and which are relaxed.}
  \label{fig:case_study1}
\end{figure}
\section{Related Work}
\label{app:related_work}
\subsection{Speculative Decoding for LLMs}
Speculative decoding is shown to be an effective approach to accelerate LLMs while maintaining the original output distribution, following a draft-then-verify paradigm.
Initial explorations~\citep{DBLP:conf/icml/LeviathanKM23,DBLP:conf/nips/KimMMMMGK23,DBLP:conf/emnlp/Xia0WCWS23} attempt to use existing small LLMs as draft models to ensure reliable speculation. Self-speculative methods~\citep{xia2025swift,zhang2024draft,DBLP:journals/corr/abs-2505-16162} use partial layers of the original model to generate predictions, without introducing extra models. Retrieval-based methods~\citep{DBLP:conf/naacl/0012ZCL024,DBLP:conf/acl/LuoWZZZ0X25} retrieve $n$-gram continuations as draft tokens.
Long-context SD methods~\citep{DBLP:journals/corr/abs-2502-17421,chen2024magicdec,sun2024triforce} target at long-context scenario.
Tree-based speculation methods~\citep{DBLP:conf/icml/LiW0024,DBLP:conf/emnlp/LiW0024,DBLP:conf/asplos/MiaoOZCWZWZYSSC24} are proposed to boost the average accept length, by predicting multiple candidates and forming draft trees. 
Recent advancements~\citep{DBLP:conf/icml/CaiLGPLCD24,DBLP:conf/icml/LiW0024,DBLP:conf/icml/Du0XWY0LXNTY24,DBLP:conf/iclr/ZhangWHX25} focus on enhancing the efficiency and accuracy of the draft model through specialized training, and performing parallel drafting and verification~\cite{liupearl,shen2026doublebreakingaccelerationlimit,kumar2026speculative}.

\paragraph{Loosely Speculative Decoding}
Strict verification often rejects semantically valid drafts, limiting achievable speedups. To address this, recent advances propose “loosely” variants to relax the criterion.
JudgeDecoding~\cite{DBLP:journals/corr/abs-2504-20039} trains a compact module on top of the embeddings to
produce ``judgement'' of the current continuation.
Following studies~\cite{DBLP:journals/corr/abs-2505-18629,zhong2025speeding,DBLP:journals/corr/abs-2505-13204,DBLP:conf/acl/HolsmanHD25,DBLP:journals/corr/abs-2510-12966} further enrich the area of loose verification.
Recently, \textsc{FLy}~\cite{li2025looselysd} proposes a training-free relaxation mechanism based on the entropy of mismatched tokens and whether the subsequent outputs remain unchanged within a predefined window.
These efforts reflect a rising interest in moving beyond strict token-level equivalence, strategically trading exactness for improved efficiency.

\subsection{Speculative Decoding for LVLMs}
Large Vision-Language Models (LVLMs) suffer from slow inference due to their large parameter counts and long input sequences. To address this, prior work~\citep{ji2025specvlm,DBLP:journals/corr/abs-2509-15235,DBLP:journals/corr/abs-2505-10526,DBLP:journals/corr/abs-2509-11961,DBLP:journals/corr/abs-2510-22641,DBLP:journals/corr/abs-2509-23928,DBLP:journals/corr/abs-2509-11815,DBLP:journals/corr/abs-2505-14260,DBLP:journals/corr/abs-2505-12728,zhang2026sparrow,kong2026parallelvlm,shen2026mmspec} has explored speculative decoding for LVLMs, primarily through vision-aware drafting.
For Video-LLMs, SpecVLM~\cite{ji2025specvlm} pioneers in adopting a training-free method, which performs token pruning for the draft model to address the KV cache bottleneck in long video scenario.
\method is the first to perform visual-aware verification, effectively breaking the performance ceiling imposed by the \textit{exact-match} for Video-LLMs.
\section{Conclusion}


We present \method, the first loosely speculative decoding tailored for Video-LLMs. Thanks to the importance disparity of generated tokens we identified in video understanding, \method leverages visual-semantic guidance to strictly verify visual-relevant tokens while relaxing criteria for visual-irrelevant ones, effectively overcoming the rigid token-level matching bottleneck. 
Ultimately, \method breaks the conventional acceleration ceiling of visual speculative decoding, providing the community with a practical foundation for more efficient and responsive video understanding in real-world applications.

\clearpage
\section{Limitations}
While our loosely speculative decoding framework yields substantial gains in mean accepted length and acceleration for Video-LLMs, several limitations merit consideration. First, since our approach is in a training-free manner, it may require tuning hyper-parameters to better balance output quality and speedup ratio when generalizing to tasks with substantially higher or lower visual information density. 
Second, our method primarily focuses on the mainstream visually descriptive captioning and QA tasks. For complex visual reasoning tasks that involve more logical tokens, mismatches that are logically important yet visually irrelevant may be potentially handled using a more careful strategy (e.g., further reducing $\lambda$ or adding auxiliary heuristics).
Third, the position shift-tolerant (\textsc{PST}) component is a bonus module tailored to the output patterns of existing video understanding tasks, relying on cases where the draft and target models exhibit positionally shifted outputs. However, since its impact is small, it can be enabled selectively in other settings.
Finally, exploring tighter integration with more draft models that are finetuned or specially trained for LVLMs~\citep{DBLP:journals/corr/abs-2509-15235,DBLP:journals/corr/abs-2509-23928}, and extending our framework to more visual understanding tasks are also left for future work.
\section{Ethical Considerations}
All experiments in this work are conducted using open-source datasets and models~\footnote{\url{https://huggingface.co/datasets/wchai/Video-Detailed-Caption} (Apache License 2.0)
}
\footnote{\url{https://huggingface.co/datasets/Enxin/MovieChat-1K-test} (BSD 3-Clause License)
}
\footnote{\url{https://huggingface.co/datasets/lmms-lab/Video-MME} (Creative Commons Attribution-ShareAlike 4.0 International License)
}
\footnote{\url{https://huggingface.co/datasets/lmms-lab/VideoDetailCaption} (CC-BY-4.0 License)
}.
Our research focuses solely on improving inference efficiency and does not involve any sensitive data, human subjects, or commercial use.


\bibliography{custom}
\clearpage
\appendix
\crefalias{section}{appendix}
\crefalias{subsection}{appendix}

\section{Theoretical Derivations for Section~\ref{sec:theoretical_analysis}}
\label{app:strict_proof}

\subsection{Geometric Acceptance Assumption}
\label{app:assumption}

\paragraph{Assumption}(Geometric Acceptance Dynamics).
The generation process is inherently autoregressive, where the alignment probability is conditioned on the prefix $P(\hat{y}_k = y_k \mid y_{<k})$. Modeling these step-wise dependencies renders the derivation intractable. Thus, for \textbf{theoretical tractability}, we simplify the analysis by approximating token correctness as independent and identically distributed (i.i.d.) events with a base alignment accuracy $\alpha \coloneqq P(\hat{y}_k = y_k)$. Consequently, the acceptance process follows a geometric distribution governed by the failure rate $\epsilon \coloneqq 1 - \alpha$.

\subsection{Proof of Proposition~\ref{prop:strict_bound}}
\label{app:proof_proposition_2.2}

\begin{proof}
We analyze the expected value of the accepted length $\tau^{\triangledown}_{\text{strict}}$, treating it as a discrete random variable defined over the domain $\{0, 1, \dots, K\}$.
Invoking the tail sum formula for the expectation of non-negative integer variables, we express $\mathbb{E}[\tau^{\triangledown}_{\text{strict}}]$ as the sum of cumulative probabilities:
\begin{equation}
\label{eq:P1}
\mathbb{E}[\tau^{\triangledown}_{\text{strict}}] = \sum_{k=1}^{K} P(\tau^{\triangledown}_{\text{strict}} \ge k).
\end{equation}
The event $\tau^{\triangledown}_{\text{strict}} \ge k$ is logically equivalent to the successful verification of the first $k$ consecutive tokens (i.e., the prefix of length $k$ matches). Under the i.i.d. assumption with alignment accuracy $\alpha$, this probability factorizes as:
\begin{equation}
\label{eq:P2}
P(\tau^{\triangledown}_{\text{strict}} \ge k) = \prod_{j=1}^{k} P(\hat{y}_j = y_j) = \alpha^k.
\end{equation}
Substituting~\Cref{eq:P2} into~\Cref{eq:P1} reveals a finite geometric series characterized by the first term $\alpha$ and the common ratio $\alpha$. Applying the standard closed-form summation formula, we derive the exact expectation:
\begin{equation}
\mathbb{E}[\tau^{\triangledown}_{\text{strict}}] = \sum_{k=1}^{K} \alpha^k = \frac{\alpha(1 - \alpha^K)}{1 - \alpha}.
\end{equation}
Crucially, since $\alpha \in (0, 1)$, this is a \textbf{series of strictly positive terms}. This implies that the partial sum is \textbf{strictly monotonically increasing} with respect to $K$. Consequently, for any finite draft size $K$, the expectation is strictly bounded by the infinite series sum (i.e., the partial sum is strictly less than the infinite sum):
\begin{equation}
\mathbb{E}[\tau^{\triangledown}_{\text{strict}}] < \sum_{k=1}^{\infty} \alpha^k = \frac{\alpha}{1 - \alpha}.
\end{equation}
We now relate this bound to the failure rate $\epsilon = 1 - \alpha$. By substituting $\alpha = 1 - \epsilon$, we derive the limit value:
\begin{equation}
\frac{\alpha}{1 - \alpha} = \frac{1 - \epsilon}{\epsilon} = \frac{1}{\epsilon} - 1.
\end{equation}
Observing that $\frac{1}{\epsilon} - 1$ is strictly less than $\frac{1}{\epsilon}$ (since $\epsilon > 0$), we conclude with the strict upper bound:
\begin{equation}
\mathbb{E}[\tau^{\triangledown}_{\text{strict}}] < \frac{1}{\epsilon}.
\end{equation}
This confirms that the performance of strictly SD is theoretically bottlenecked by the inverse of the draft model's failure rate.
\end{proof}

\subsection{Proof of Theorem~\ref{thm:sparsity_scaling}}
\label{app:proof_theorem_2.3}

\begin{proof}
Under the proposed mechanism, tokens in $\mathcal{A}$ are accepted with probability $\approx 1$ (conditioned on syntactic plausibility), while tokens in $\mathcal{V}$ retain the strict acceptance probability $\alpha$. This transforms the uniform alignment probability $\alpha$ into an \textbf{effective alignment rate} $\tilde{\alpha}$, defined as the weighted average over the sequence:
\begin{equation}
\tilde{\alpha} \coloneqq \underbrace{\rho \cdot \alpha}_{\triangledown_\text{strict} \text{ for } \mathcal{V}} + \underbrace{(1 - \rho) \cdot 1}_{\triangledown_\text{loose} \text{ for } \mathcal{V}'} = 1 - \rho(1 - \alpha).
\end{equation}
Analogous to the derivation in Proposition~\ref{prop:strict_bound}, the expected accepted length follows the geometric expectation formula. By expanding the algebraic term, we establish a strict upper bound:
\begin{equation}
\begin{aligned}
\mathbb{E}[\tau^{\triangledown}_{\text{loose}}] &= \sum_{k=1}^{K} \tilde{\alpha}^k \approx \frac{\tilde{\alpha}}{1 - \tilde{\alpha}} \\
&= \frac{1 - \rho(1 - \alpha)}{1 - (1 - \rho(1 - \alpha))} \\
&= \frac{1 - \rho\epsilon}{\rho\epsilon} = \frac{1}{\rho\epsilon} - 1 < \frac{1}{\rho\epsilon}.
\end{aligned}
\end{equation}
This approximation holds because the visual density is sparse ($\rho \ll 1$), implying that the term $\frac{1}{\rho\epsilon}$ dominates the subtractive constant. 
Comparing this to the strict expectation limit $\mathbb{E}[\tau^{\triangledown}_{\text{strict}}] \approx \frac{1}{\epsilon}$, we derive the scaling ratio:

\begin{equation}
\frac{\mathbb{E}[\tau^{\triangledown}_{\text{loose}}]}{\mathbb{E}[\tau^{\triangledown}_{\text{strict}}]} \approx \frac{1/\rho\epsilon}{1/\epsilon} = \frac{1}{\rho}.
\end{equation}
This concludes the proof.
\end{proof}

\begin{table*}[!t]
  \centering
  \setlength{\tabcolsep}{4.5pt}
  \renewcommand{\arraystretch}{1.15}
  \begin{tabular}{@{}l l l c c@{}}
    \toprule
    \textbf{Benchmark} & \textbf{Type} & \textbf{Output} & \textbf{Frames} & \textbf{\#Instances} \\
    \midrule
    VDC~\cite{DBLP:conf/iclr/ChaiSDMMBHXM25}       & Captioning & Long   & 64  & 120 \\
    VDD~\cite{vdc2024}                            & Captioning & Long   & 128 & 30  \\
    MovieChat~\cite{DBLP:conf/cvpr/SongCWZZWCG0ZLH24}            & QA         & Medium & 128 & 100 \\
    Video-MME~\cite{DBLP:conf/cvpr/FuDLLRZWZSZCLLZ25}                  & QA         & Short  & 64  & 500 \\
    \bottomrule
  \end{tabular}
  \caption{Benchmark details.}
  \label{tab:benchmark_detail}
\end{table*}


\section{Benchmark Details}
\label{app:benchmark_detail}
For comprehensive evaluation on video understanding, we conduct experiments on both video captioning and video question answering tasks, where the output length ranges from a few tokens to long captions.
All evaluations follow the standard implementation of \texttt{LMMs-Eval}~\cite{lmms_eval2024} and use Large Language Models as judges. Prompt template aligns with step 3 in~\Cref{fig:prompt_template}.
The video captioning benchmarks include 
\ding{172} Video Detail Caption (VDC)~\cite{DBLP:conf/iclr/ChaiSDMMBHXM25}, which covers main objects, backgrounds, details, and camera tests, as well as  \ding{173} Video Detail Description (VDD)~\cite{vdc2024}, which evaluates holistic video-level descriptions.
The video question answering benchmarks include \ding{174} MovieChat~\cite{DBLP:conf/cvpr/SongCWZZWCG0ZLH24}, which evaluates open-ended question answering over videos, and \ding{175} Video-MME~\cite{DBLP:conf/cvpr/FuDLLRZWZSZCLLZ25}, which assesses multiple-choice question answering.
These benchmarks span diverse task types, video lengths, and output lengths, enabling a comprehensive evaluation of both acceleration and performance on video understanding.
We summarize the generation length, number of frames sampled, and number of instances of each benchmark in~\Cref{tab:benchmark_detail}.

\section{Baseline Details}
\label{app:baseline_detail}
For our re-implementation of \textsc{Fly}~\cite{li2025looselysd}, we set the size of the deferred window $\mathrm{Window}=4$, matching the ratio of $\mathrm{Window}$ to draft length used in the original paper. 
According to model-specific properties, we set the gate threshold of entropy $\mathrm{Entropy}=0.1$ for \texttt{Qwen2.5-VL-32B} and $\mathrm{Entropy}=0.2$ for the other models. The threshold is chosen to induce a relaxation level comparable to that of \method.
In~\Cref{sec:ablation_effect_of_visual}, we evaluate more hyperparameter configurations of \textsc{FLy} and $\textsc{FLy}^{\ominus}$, and \method consistently remains superior in both speedup and performance (see~\Cref{fig:effective}).

For the implementation of \textsc{SpecVLM}~\cite{ji2025specvlm}, we follow the original paper and adopt the same tree structure and pruning ratio $=0.9$ for visual tokens.
The visual tokens are uniformly pruned, and the same attention implementation (scaled-dot-product-attention) is ensured.

\begin{table*}[!t]
  \centering
  \setlength{\tabcolsep}{6pt}
  \renewcommand{\arraystretch}{1.0}
  \begin{tabular}{@{}l|cc|cc@{}}
    \toprule
    \multirow{2}{*}{Metric} &
    \multicolumn{2}{c|}{\texttt{Std.-SD-Qwen2.5-VL}} &
    \multicolumn{2}{c}{\texttt{Std.-SD-LLaVA-OV}} \\
    & w/o Draft Tree & w/ Draft Tree
    & w/o Draft Tree & w/ Draft Tree \\
    \midrule
    Speedup & 2.70$\times$ & 2.97$\times$ & 2.94$\times$ & 3.30$\times$ \\
    Retention & 99.6\% & 95.5\% & 99.8\% & 96.5\% \\
    \boldmath$\tau$ & 7.76 & 8.94 & 7.34 & 8.46 \\
    \bottomrule
  \end{tabular}
  \caption{Experiment of \method with draft token trees on the VDC benchmark.}
  \label{tab:draft_tree}
\end{table*}


\section{Exploration on Draft Model Sizes}
\label{app:draft_model_size}

\begin{figure}[t]
  \centering
  \includegraphics[width=0.48\textwidth]{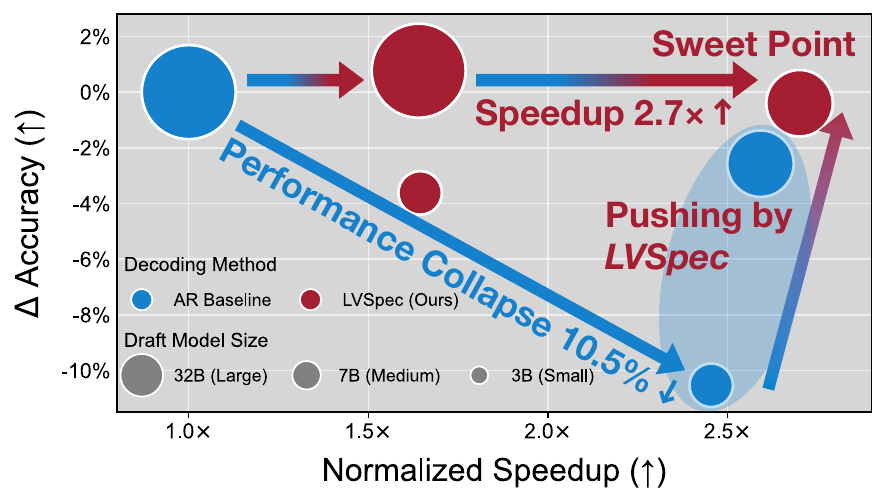}
  \caption{Evaluation of inference efficiency and generation quality. We plot the $\Delta$ Accuracy against the speedup ratio on the VDC task. Bubble sizes correspond to the draft model scales (3B, 7B, and 32B) within the Qwen2.5-VL family. The proposed LVSpec (red bubbles) demonstrates robust performance scaling, effectively mitigating the accuracy loss observed in the autoregressive decoding (AR) baseline (blue bubbles) even when using weaker draft models for SD.}
  \label{fig:draft_model_size}
\end{figure}

To assess the generalization of \method under different draft model capacities, we experiment with both larger draft model \texttt{Qwen2.5-VL-32B}, medium draft model \texttt{Qwen2.5-VL-7B}, and smaller draft model \texttt{Qwen2.5-VL-3B}, fixing \texttt{Qwen2.5-VL-32B} as the target model on the VDC task, as shown in~\Cref{fig:draft_model_size}. For draft model \texttt{Qwen2.5-VL-32B} and \texttt{Qwen2.5-VL-7B}, we keep $\lambda=0.7$, and \method still maintains nearly lossless accuracy value ($+0.77\%$ and $-0.11\%$) while achieving a speedup ratio of $1.64\times$ and $2.70\times$ . 
For draft model \texttt{Qwen2.5-VL-3B}, we adopt a smaller loose parameter $\lambda=0.3$ and $N=5$ since the draft quality degrades substantially. Even so, the results demonstrate that \method can adaptively preserve accuracy across different model pairings, while delivering meaningful acceleration.

\section{Exploration on Draft Token Trees}
\label{app:draft_tree}


\method operates on the verification mechanism and is compatible with the draft tree technique commonly used in speculative decoding~\citep{DBLP:conf/icml/LiW0024,DBLP:conf/emnlp/LiW0024,DBLP:conf/asplos/MiaoOZCWZWZYSSC24}. We adopt a simple static tree with two branches (each with draft tokens $K=10$) and perform loose verification along each branch, finally selecting the branch with the longest accepted length. \Cref{tab:draft_tree} shows that incorporating the tree structure further improves the speedup from 2.70$\times$ and 2.94$\times$ to 2.97$\times$ and 3.30$\times$. However, it in turn increases the level of relaxation, leading to a modest loss in accuracy. We leave the question of how to balance acceptance and speedup across multiple tree branches as future work.

\section{Computation Breakdown}
\begin{figure}[!t]
  \centering
  \includegraphics[width=0.48\textwidth]{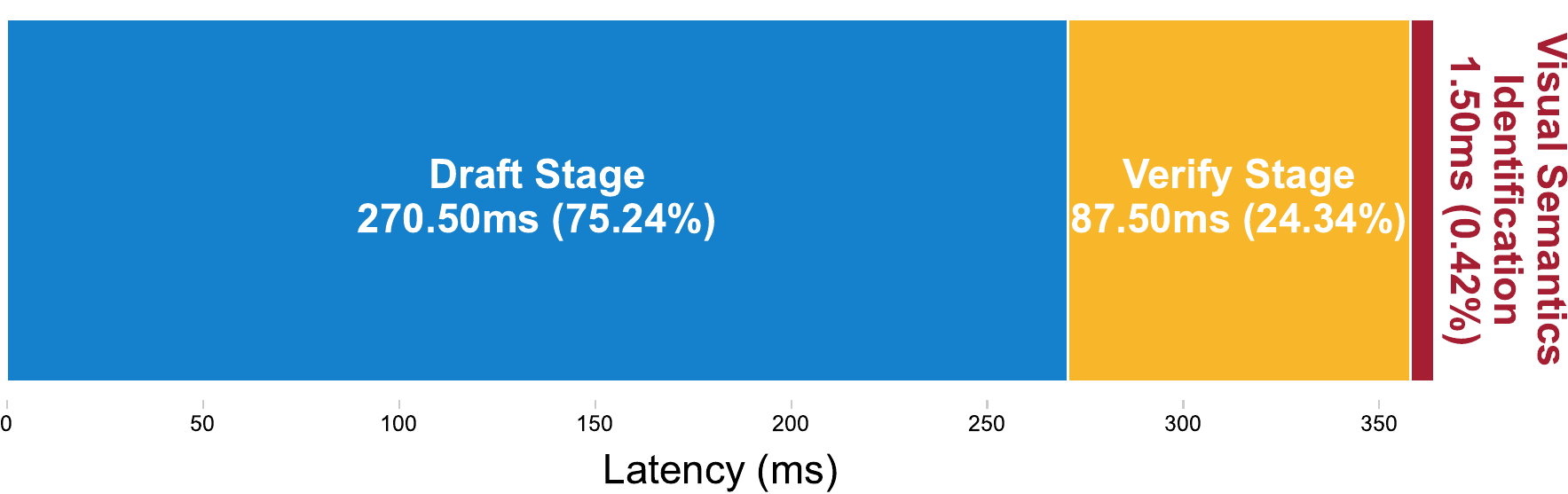}
  \caption{The computational overhead of \method. Latency is tested using \texttt{Std.-SD-Qwen2.5-VL} and VDC task, on two NVIDIA H200 GPUs.}
  \label{fig:computation_breakdown}
\end{figure}
In each decoding step, \method requires computing the visual relevance of text tokens. We report the overhead of this computation in~\Cref{fig:computation_breakdown}. Compared with the cost of drafting and verification, the visual relevance computation accounts for only 0.42\% of the runtime, indicating minimal impact on end-to-end latency. 
Since the cost of computing cosine similarity scales only linearly with the sequence length, this overhead does not become larger relative to the computation of other components as the video length increases.


\section{More Experimental Results}
\label{app:results}
\paragraph{Performance retention results.} In~\Cref{fig:accuracy}, we show the retention results of the performance metric (accuracy and score) for different tasks and model settings. \Cref{tab:accuracy} provides the raw data.
\paragraph{More comparison results with \textsc{SpecVLM}.}
In the main experiment (\Cref{tab:speedup}), we compare \method with \textsc{SpecVLM}~\cite{ji2025specvlm} equipped with the same draft tree as in the original paper.
Since the prior loosely SD method~\cite{li2025looselysd} aims to increase the mean accepted length, it by default use a long chain-structured draft sequence without forming a draft tree. 
We therefore conduct experiments that compare \method against \textsc{SpecVLM} under the same draft sequence ($K=10$) in~\Cref{tab:specvlm_without_tree}.
Here, \method and \textsc{SpecVLM} ($K=10$) differ only in the verification criterion. The comparison shows that: (i) \method surpasses \textsc{SpecVLM} with the same draft structure in both mean accepted length ($+128\%$) and speedup ($+93\%$) solely by loosing the verification. (ii) \textsc{SpecVLM}’s acceleration relies heavily on the draft tree structure, whereas \method breaks this dependency and is able to accept more tokens with fewer draft nodes, showing potential for larger batch sizes. (iii) The empirical results on mean accepted length closely align with our theoretical analysis in ~\Cref{sec:theoretical_analysis}. Specifically, \method uses $\lambda=0.7$ to approximate the setting with visual density $\rho=0.3$ in~\Cref{thm:sparsity_scaling}. When \textsc{SpecVLM} ($K=10$) attains $\tau=3.41$, corresponding to a failure rate of 65.9\%, \method achieves a failure rate of $1 - 7.76/10 = 22.4\%$, which closely matches the diluted estimate $65.9\% \times 0.3 = 19.8\%$.

\section{Details of Oracle Study}
\label{app:oracle}
As mentioned in~\Cref{sec:empirical_sparsity}, we observe that generation quality is governed by a sparse set of “anchor” tokens, whereas high-frequency irrelevant tokens carry minimal information density. A detailed prompt template is provided in~\Cref{fig:prompt_template} and an illustration example is provided in~\Cref{fig:case_study_oracle}.

\section{Potential Improvements to \method}
\label{app:potential_improvement}
Although the loose verification strategy of \method demonstrates effective performance across different model settings and benchmarks, it is still necessary to discuss the potential improvements of \method.
We acknowledge that more complex visual reasoning tasks may involve tokens that are visually irrelevant yet logically critical (e.g., ``not''), and relaxing the verification of those tokens potentially decreases performance.
We discuss this case from the following perspectives: (i) Our method primarily aims to explore the performance boundary of loosely speculative decoding on mainstream video understanding tasks, and to provide a empirically validated solution that is effective under common model settings. 
Although such cases can arise in theory, the proportion of visually irrelevant but harmful draft tokens is small in descriptive output patterns, and the impact on overall performance is limited according to empirical results (\Cref{fig:accuracy} and~\Cref{tab:accuracy}).
(ii) Potential failure cases are a shared issue in loosely speculative decoding.
JudgeDecoding~\cite{BachmannAPGSDST25} can degrade in quality when encountering out-of-domain inputs or output patterns unseen during training. FLy~\cite{li2025looselysd} treats a current mismatch as a synonym if it is followed by several matches, which may inadvertently accept drafts that preserve superficial structural consistency but are nonetheless lower-quality than the target model’s output.
(iii) When the output patterns do contain substantial visually irrelevant yet logically critical tokens, improvements can be made upon \method. First, we can augment the visual relevance metric with additional heuristics to identify non-visual keywords. This would not undermine the role of the visual relevance metric since visually grounded tokens still remain critical. Second, when such tokens cannot be reliably identified, a simple fallback to strict verification can be used to avoid performance degradation.

With the development of complex video reasoning~\cite{fei2024video,han2025videoespresso,zhang2025mmesgbench}, incorporating adaptive video token pruning~\cite{liu2025vidcom2,zhang2026efficientinferencelargevisionlanguage}, KV cache compression~\cite{liu2025mixkv,zeng2026hybridkvhybridkvcache,zhang2026efficientinferencelargevisionlanguage}, and in-context draft capability~\cite{gao2025aim} are also promising directions.

\section{LLM Usage}
Large Language Models (LLMs) were used to aid in the code writing and manuscript polishing.
Specifically, the usage includes refining the language, improving readability, and ensuring clarity in the paper.
It is important to note that LLMs were not involved in the ideation, research methodology, or experimental design.

\begin{table*}[!ht]
  \centering
  \setlength{\tabcolsep}{6pt}
  \renewcommand{\arraystretch}{1.15}
  \resizebox{\textwidth}{!}{
  \begin{tabular}{c|l|cc|c|cc|cc}
    \toprule
    \multirow{2}{*}{\textbf{Model}} & \multirow{2}{*}{\textbf{Method}} &
    \multicolumn{2}{c|}{\textbf{VDC}} &
    \textbf{VDD} &
    \multicolumn{2}{c|}{\textbf{MovieChat}} &
    \multicolumn{2}{c}{\textbf{VideoMME}} \\
    & & Score & Acc(\%) & Score & Score & Acc(\%) & Acc$_{short}$(\%) & Acc$_{long}$(\%) \\
    \midrule

    \multirow{4}{*}{\cellcolor{white}\rotatebox{90}{\makecell[c]{\ttfamily Std.-SD\\ \ttfamily Qwen2.5-VL}}}
    & \gray{Autoregressive}  &
    \gray{1.93}   & \gray{31.27} &
    \gray{3.98}   &
    \gray{2.95}   & \gray{54.0}  &
    \gray{67.6}   & \gray{61.8} \\
    & $\textsc{FLy}^{\ominus}$ &
    1.87 & 29.82 &
    3.46 &
    3.11 & 57.0 &
    67.6 & 61.8 \\
    & \textsc{FLy} &
    1.94 &  30.90 &
    3.61 &
    2.82 & 50.0 &
    67.6 & 61.8 \\
    & \hanredcell{\method (Ours)} &
    \hanredcell{1.92} & \hanredcell{31.16} &
    \hanredcell{3.71} &
    \hanredcell{3.06} & \hanredcell{59.0}  &
    \hanredcell{67.6} & \hanredcell{61.8} \\
    \hline

    \multirow{4}{*}{\cellcolor{white}\rotatebox{90}{\makecell[c]{\ttfamily Self-SD\\ \ttfamily Qwen2.5-VL}}}
    & \gray{Autoregressive}  &
    \gray{1.80} & \gray{28.70} &
    \gray{3.55} &
    \gray{3.08} & \gray{57.0}  &
    \gray{60.6} & \gray{59.0} \\
    & $\textsc{FLy}^{\ominus}$ &
    1.77 & 27.46 &
    3.70 &
    2.90 & 51.0 &
    60.6 & 59.0 \\
    & \textsc{FLy} &
    1.84 & 29.81 &
    3.63 &
    3.12 & 58.0 &
    60.6 & 59.0 \\
    & \hanredcell{\method (Ours)} &
    \hanredcell{1.82} & \hanredcell{27.97} &
    \hanredcell{3.87} &
    \hanredcell{3.04} & \hanredcell{57.0}  &
    \hanredcell{60.6} & \hanredcell{59.0} \\
    \hline

    \multirow{4}{*}{\cellcolor{white}\rotatebox{90}{\makecell[c]{\ttfamily Std.-SD\\ \ttfamily LLaVA-OV}}}
    & \gray{Autoregressive}  &
    \gray{1.84} & \gray{28.66} &
    \gray{3.50} &
    \gray{2.60} & \gray{43.0}  &
    \gray{75.0} & \gray{65.8} \\
    & $\textsc{FLy}^{\ominus}$ &
    1.87 & 29.12 &
    3.53 &
    2.68 & 48.0 &
    75.0 & 65.8 \\
    & \textsc{FLy} &
    1.79 & 27.53 &
    3.03 &
    2.63 & 43.0 &
    75.0 & 65.8 \\
    & \hanredcell{\method (Ours)} &
    \hanredcell{1.83} & \hanredcell{28.61} &
    \hanredcell{3.46} &
    \hanredcell{2.53} & \hanredcell{45.0}  &
    \hanredcell{75.0} & \hanredcell{65.8} \\
    \hline

    \multirow{4}{*}{\cellcolor{white}\rotatebox{90}{\makecell[c]{\ttfamily Self-SD\\ \ttfamily LLaVA-OV}}}
    & \gray{Autoregressive}  &
    \gray{1.79} & \gray{27.70} &
    \gray{3.57} &
    \gray{2.67} & \gray{48.0}  &
    \gray{67.0} & \gray{54.8} \\
    & $\textsc{FLy}^{\ominus}$ &
    1.84 & 28.83 &
    3.50 &
    2.58 & 47.0 &
    67.0 & 54.8 \\
    & \textsc{FLy} &
    1.87 & 28.90 &
    3.63 &
    2.61 & 47.0 &
    67.0 & 54.8 \\
    & \hanredcell{\method (Ours)} &
    \hanredcell{1.83} & \hanredcell{28.38} &
    \hanredcell{3.50} &
    \hanredcell{2.83} & \hanredcell{51.0}  &
    \hanredcell{67.0} & \hanredcell{54.8} \\
    \bottomrule

  \end{tabular}
  }
  \caption{
  Raw data of performance retention results in~\Cref{fig:accuracy}. The overall retention ratio~in~\Cref{fig:accuracy} is computed on average of each performance metric.
  }
  \label{tab:accuracy}
\end{table*}

\setlength{\tabcolsep}{6pt}          
\renewcommand{\arraystretch}{1.5}   
\setlength{\extrarowheight}{1.2pt}   

\begin{table*}[t]
  \centering
  \resizebox{\textwidth}{!}{
  \begin{tabular}{l|l|ccc|ccc|ccc}
    \toprule
    \multicolumn{1}{l|}{\multirow{2}{*}{\textbf{Model}}} &
    \multicolumn{1}{l|}{\multirow{2}{*}{\textbf{Method}}} &
    \multicolumn{3}{c|}{\textbf{Video Detail Caption}} &
    \multicolumn{3}{c|}{\textbf{Video Detail Description}} &
    \multicolumn{3}{c}{\textbf{MovieChat}} \\
    & &
    $\tau$ & Token/s & Speedup &
    $\tau$ & Token/s & Speedup &
    $\tau$ & Token/s & Speedup \\
    \midrule

    \multirow{3}{*}{\rotatebox{90}{\makecell[c]{\ttfamily Std.-SD \\ \ttfamily Qwen2.5-VL}}}
    & \textsc{SpecVLM} ($K=10$)
      & 3.41 & 11.15 & 1.40$\times$
      & 3.33 & 11.85 & 1.16$\times$
      & 3.02 & 10.68 & 1.09$\times$ \\
    & \textsc{SpecVLM}
      & 3.29 & 15.90 & 2.00$\times$
      & 3.31 & 17.96 & 1.76$\times$
      & 3.19 & 17.05 & 1.74$\times$ \\
    & \hanredcell{\method (Ours)}
      & \hanredcell{\textbf{7.76}} & \hanredcell{\textbf{21.47}} & \hanredcell{\textbf{2.70$\times$}}
      & \hanredcell{\textbf{7.68}} & \hanredcell{\textbf{23.58}} & \hanredcell{\textbf{2.34$\times$}}
      & \hanredcell{\textbf{7.74}} & \hanredcell{\textbf{21.49}} & \hanredcell{\textbf{2.19$\times$}} \\
    \hline

    \multirow{3}{*}{\rotatebox{90}{\makecell[c]{\ttfamily Self-SD \\ \ttfamily Qwen2.5-VL}}}
    & \textsc{SpecVLM} ($K=10$)
      & 5.31 & 25.99 & 1.26$\times$
      & 5.49 & 26.03 & 1.05$\times$
      & 4.63 & 21.77 & 0.85$\times$ \\
    & \textsc{SpecVLM}
      & 3.85 & 31.45 & 1.53$\times$
      & 3.89 & 35.00 & 1.41$\times$
      & 3.66 & 32.20 & 1.26$\times$ \\
    & \hanredcell{\method (Ours)}
      & \hanredcell{\textbf{8.36}} & \hanredcell{\textbf{36.49}} & \hanredcell{\textbf{1.77$\times$}}
      & \hanredcell{\textbf{8.64}} & \hanredcell{\textbf{39.32}} & \hanredcell{\textbf{1.59$\times$}}
      & \hanredcell{\textbf{7.99}} & \hanredcell{\textbf{34.04}} & \hanredcell{\textbf{1.33$\times$}} \\
    \hline

    \multirow{3}{*}{\rotatebox{90}{\makecell[c]{\ttfamily Std.-SD \\ \ttfamily LLaVA-OV}}}
    & \textsc{SpecVLM} ($K=10$)
      & 3.66 & 11.61 & 1.75$\times$
      & 4.02 & 10.51 & 1.75$\times$
      & 4.74 & 5.49  & 1.24$\times$ \\
    & \textsc{SpecVLM}
      & 3.37 & 15.75 & 2.38$\times$
      & 3.53 & 14.94 & 2.49$\times$
      & 3.77 & 8.45  & 1.92$\times$ \\
    & \hanredcell{\method (Ours)}
      & \hanredcell{\textbf{7.34}} & \hanredcell{\textbf{19.44}} & \hanredcell{\textbf{2.94$\times$}}
      & \hanredcell{\textbf{7.59}} & \hanredcell{\textbf{17.45}} & \hanredcell{\textbf{2.91$\times$}}
      & \hanredcell{\textbf{8.10}} & \hanredcell{\textbf{12.16}} & \hanredcell{\textbf{2.75$\times$}} \\
    \hline

    \multirow{3}{*}{\rotatebox{90}{\makecell[c]{\ttfamily Self-SD \\ \ttfamily LLaVA-OV}}}
    & \textsc{SpecVLM} ($K=10$)
      & 5.47 & 26.27 & 0.99$\times$
      & 5.57 & 22.61 & 1.28$\times$
      & 6.98 & 14.16 & 0.82$\times$ \\
    & \textsc{SpecVLM}
      & 3.91 & 35.46 & 1.34$\times$
      & 3.95 & 29.08 & 1.65$\times$
      & 4.17 & \textbf{21.19} & \textbf{1.23$\times$} \\
    & \hanredcell{\method (Ours)}
      & \hanredcell{\textbf{8.46}} & \hanredcell{\textbf{36.40}} & \hanredcell{\textbf{1.38$\times$}}
      & \hanredcell{\textbf{8.87}} & \hanredcell{\textbf{34.96}} & \hanredcell{\textbf{1.98$\times$}}
      & \hanredcell{\textbf{8.77}} & \hanredcell{20.28} & \hanredcell{1.18$\times$} \\
    \bottomrule
  \end{tabular}
  }
  \caption{
  Speedup ratios and mean accepted length $\tau$. \textsc{SpecVLM} ($K=10$) and \method adopt identical draft structure for fair comparison, while \textsc{SpecVLM} adopts a draft tree as in its original paper for optimization ($depth=5$). Our \method outperforms \textsc{SpecVLM} ($K=10$) by up to 93\% and \textsc{SpecVLM} by up to 43\% on speedup ratio.
  }
  \label{tab:specvlm_without_tree}
\end{table*}
\begin{figure*} [t!] 
\centering
    \includegraphics[scale=1.0]{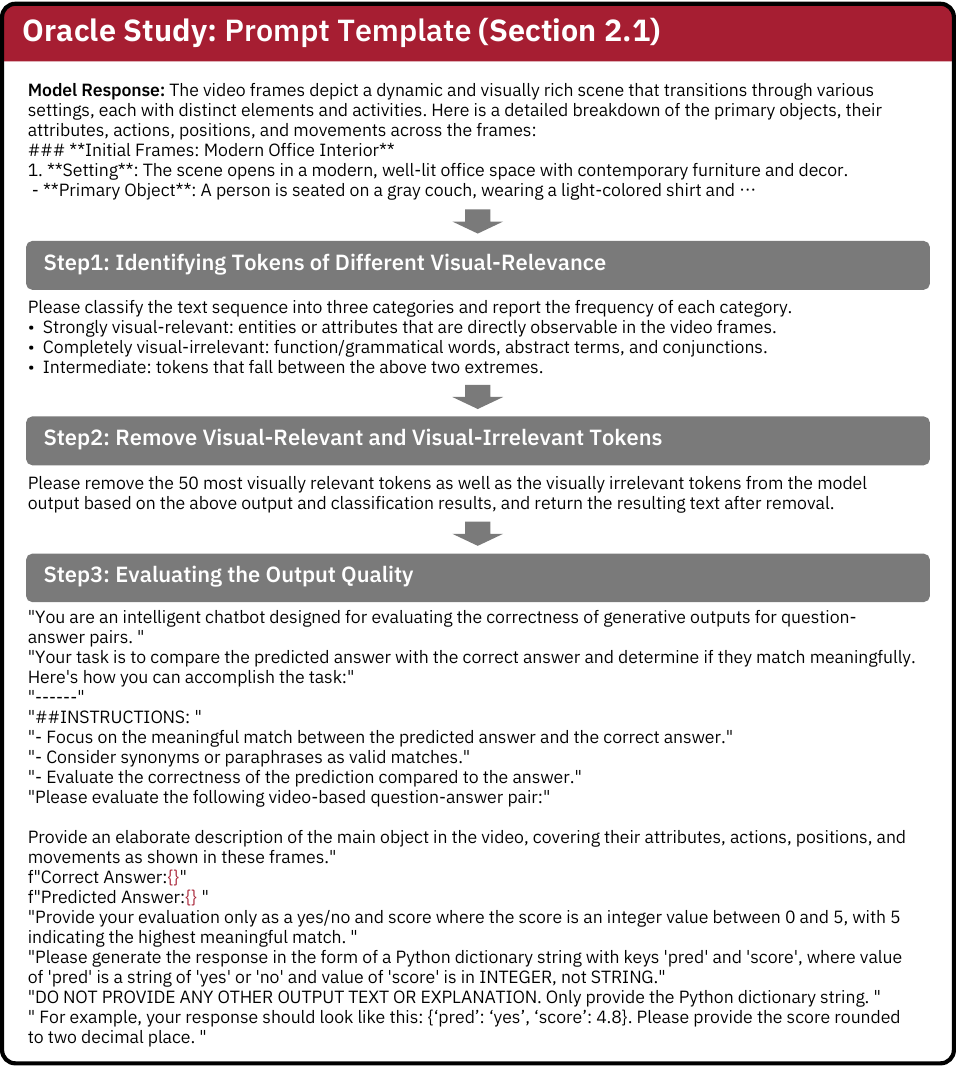}
    \caption{
    Prompt template of the oracle study in~\Cref{sec:empirical_sparsity}.
    }
    \label{fig:prompt_template}
\end{figure*}

\begin{figure*} [t!] 
\centering
    \includegraphics[scale=1.0]{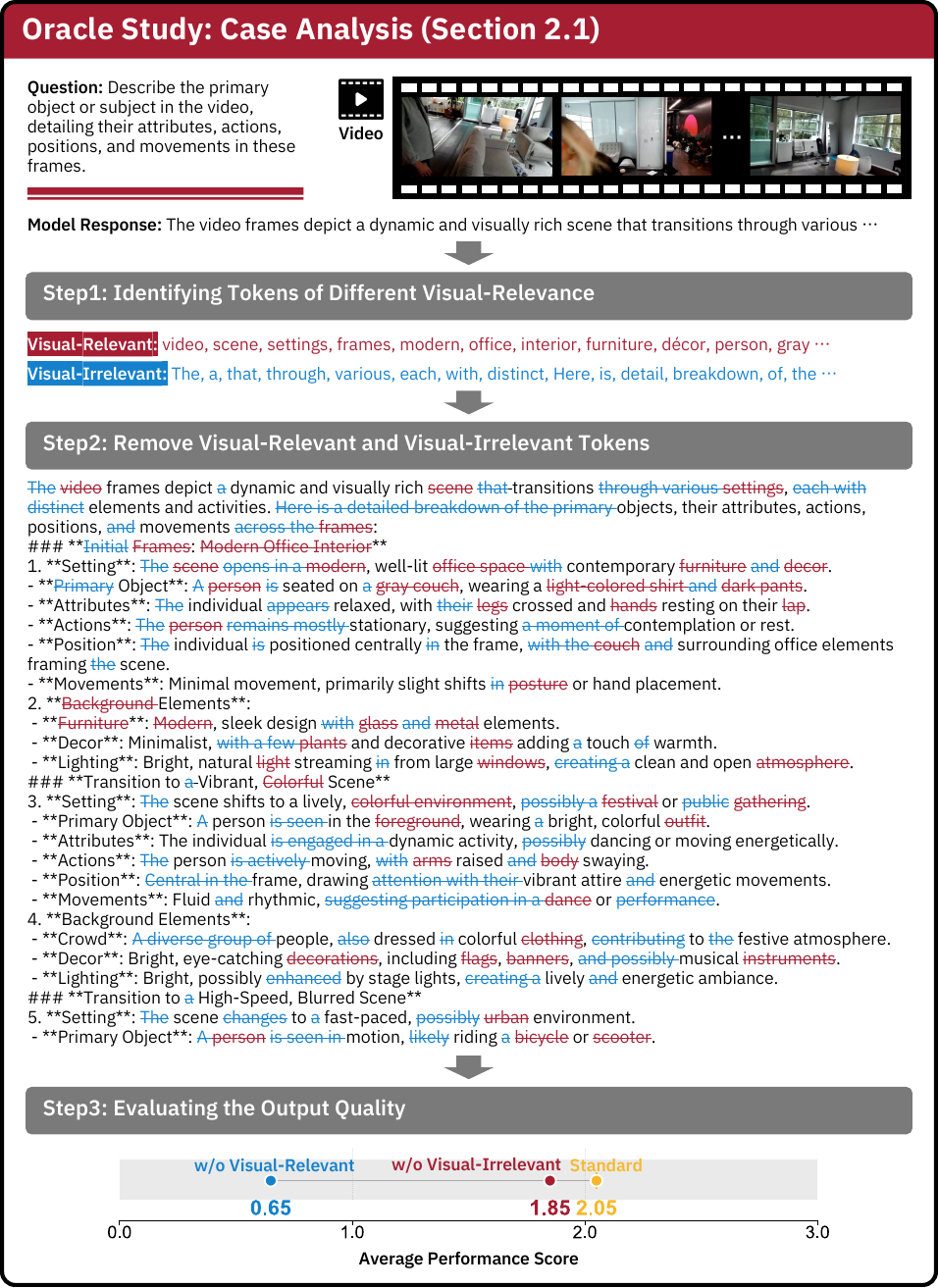}
    \caption{
    An example of the oracle study in~\Cref{sec:empirical_sparsity}.
    }
    \label{fig:case_study_oracle}
\end{figure*}

\end{document}